\documentclass[11pt]{article}
\usepackage[margin=1.0in,a4paper]{geometry}

\usepackage{mathtools}
\usepackage{amsmath, amssymb, bbm, natbib, graphicx, url, amsthm, subcaption, float, dsfont}
\usepackage{mathrsfs}
\usepackage{algorithm, algpseudocode}
\usepackage[mathcal]{eucal}
\usepackage{tikz}
\usetikzlibrary{arrows.meta, positioning, calc, fit}
\usepackage{verbatim}
\usepackage{bm}
\usepackage{standalone}
 \usepackage{tcolorbox}
 \newtcolorbox{assbox}{colback=black!5!white,colframe=black!75!black}
  \newtcolorbox{thmbox}{colback=red!5!white,colframe=red!75!black}

\usepackage{amsfonts}       
\usepackage{enumitem}
\usepackage{microtype}      
\usepackage{xcolor}
\usepackage[colorlinks=true,citecolor=blue]{hyperref}    
\usepackage{diagbox}   
\usepackage{tcolorbox}
\usepackage{nicefrac}

\newcommand{\RMS}{{\bar{2}}}

\newcommand{\loss}{\mathrm{loss}}
\newcommand{\Law}{\mathrm{Law}}

\newcommand{\Update}
{\boldsymbol{\mathcal{U}}}
\newcommand{\tin}{\mathrm{in}}
\newcommand{\tout}{\mathrm{out}}
\newcommand{\clip}{\mathrm{cl}}
\DeclarePairedDelimiterX{\ipD}[2]{\langle}{\rangle_{\bar  2}}{#1,#2}

\newcommand{\RR}{\mathbb{R}}
\newcommand{\NN}{\mathbb{N}}

\newcommand{\Ll}{\mathcal{L}}

\newcommand{\Nn}{\mathcal{N}}

\newcommand{\Pp}{\mathcal{P}}

\newcommand{\E}{\mathbf{E}}
\renewcommand{\P}{\mathbf{P}}

\renewcommand{\d}{\mathrm{d}}

\newcommand{\id}{\mathrm{id}}

\newcommand{\eps}{\varepsilon}

\newcommand{\Lip}{\mathrm{Lip}}

\newtheorem{theorem}{Theorem}
\newtheorem{proposition}{Proposition}[section]
\newtheorem{lemma}{Lemma}[section]

\newtheorem{remark}{Remark}[section]

\newtheorem{assumption}{Assumption}

\definecolor{ForestGreen}{cmyk}{0.91,0,0.88,0.12}
\colorlet{pierrem}{ForestGreen}

\graphicspath{ {./images/} } 

\title{
The Hidden Width of Deep ResNets:\\
Tight Error Bounds and Phase Diagram
}

\author{
L\'ena\"ic Chizat\thanks{Ecole Polytechnique Fédérale de Lausanne (EPFL), Institute of Mathematics, 1015 Lausanne, Switzerland. \texttt{lenaic.chizat@epfl.ch}}
}

\begin{document}
\maketitle
\begin{abstract}
We study the gradient-based training of large-depth residual networks (ResNets)  from standard random initializations. We show that infinite-depth ResNets behave as if they were infinitely wide, regardless of their actual width. More precisely, we obtain that with a fixed embedding dimension $D$,  the training dynamics converges to a unique \emph{Neural Mean ODE} training dynamics as the depth $L$ diverges, regardless of the scaling of the hidden width $M$.
For a residual scale $\Theta_D\big(\frac{\alpha}{LM}\big)$ with $\alpha=\Theta_D(1)$, we obtain the error bound $O_D\big(\frac{1}{L}+ \frac{1}{\sqrt{LM}}\big)$ between the model's output and its limit after a fixed number gradient of steps. In this regime, the limit exhibits maximal \emph{local} feature updates, i.e.~the Mean ODE is genuinely non-linearly parameterized.
In contrast, we show that $\alpha \to \infty$ yields a \emph{lazy ODE} regime where the Mean ODE is linearly parameterized, and we derive a convergence rate in this case as well.
We then focus on the particular case of ResNets with two-layer perceptron blocks, for which we study how these scalings depend on the embedding dimension $D$. We identify the residual scale  $O\big(\frac{\sqrt{D}}{LM}\big)$ as necessary and sufficient for maximal local feature updates. In this regime, we prove a high-probability error bound $O\big(\frac{1}{L}+ \frac{\sqrt{D}}{\sqrt{LM}}\big)$ between the ResNet and its limit after a fixed number of gradient steps.
Our convergence results rely on a novel mathematical perspective on ResNets : (i) due to the randomness of the initialization, the forward and backward pass through the ResNet behave as the stochastic approximation of certain mean ODEs, and (ii) by propagation of chaos—that is, asymptotic independence of the units—this behavior is preserved through the training dynamics. We verify empirically that all our rates are tight.
\end{abstract}

\tableofcontents

\section{Introduction}
Scaling up dataset sizes and deep learning architectures has been a key driver of the performance gains observed in recent years in artificial intelligence. However, many hyperparameters (HPs) and choices determine a model's behavior—its architecture, initial weights, training algorithm, and so on—and tuning all these for optimal performance on very large models is computationally prohibitive. In this context, the theoretical analysis of large neural networks—such as the derivation of phase diagrams with tight error estimates—offers principled ways to gain intuition on the behavior of large models and organize this search space.

In this paper, we pursue this program in the context of residual architectures, which have constituted the backbones of state-of-the-art models since~\citep{he2016deep}. In our analysis, the key HPs are the depth $L$, the embedding dimension $D$, the hidden width $M$, the layerwise initialization scales (and/or scaling factors) and learning rates (LRs). In Transformers~\citep{vaswani2017attention}, the hidden width $M$ corresponds to the feedforward width in multilayer perceptron (MLP) blocks or to the number of attention heads in attention blocks. We ask the following question:
\begin{center}
\emph{What are the large-depth ($L\to\infty$) behaviors of the training dynamics of ResNets?}
\end{center}

Prior work has associated the $L \to \infty$ limit with the Neural ODE model, but establishing this connection rigorously requires highly specific weight-tied initializations, which differ from practical setups~\citep{avelin2021neural, marion2023implicit}. Another line of work combines the large-depth ($L \to \infty$) and large-width ($M \to \infty$) limits for randomly initialized ResNets, and shows that the asymptotic dynamics is that of a \emph{Mean-Field Neural ODE}~\citep{lu2020mean}, with an approximation rate $O_D\big(\frac{1}{L} + \frac{1}{\sqrt{M}}\big)$~\citep{ding2022overparameterization}. However, taking $M \to \infty$ with $D$ fixed departs significantly from practice, where $M$ is typically comparable to $D$, so it is a priori unclear whether this limit bears any connection with practical setups.

In this paper, we show that this infinite width-and-depth limit in fact faithfully models practical architectures, because it arises as $L \to \infty$ \emph{regardless} of how $M$ scales. Unlike prior works, we exhibit the central role of the interaction between hidden width $M$ and depth $L$ towards approximating the limit. We obtain an error bound that is the sum of a ``depth-discretization'' error in $O(1/L)$---the usual error of the Euler method---and a novel ``sampling error'' that follows the Monte-Carlo rate $O_D(\alpha/\sqrt{ML})$ with \emph{effective width} $ML$ where $\alpha$ is a variance term that depends on the choice of HP scaling. 
The convergence of the trained ResNet to the infinite width-and-depth model is illustrated on Figure~\ref{fig:cover} in a setting where the hidden width is $M=1$. The convergence rates are shown on Figure~\ref{fig:experiment1} (see Figure~\ref{fig:experiment-D} for the dependency in $D$).
\begin{figure}
\centering
\includegraphics[scale=0.6]{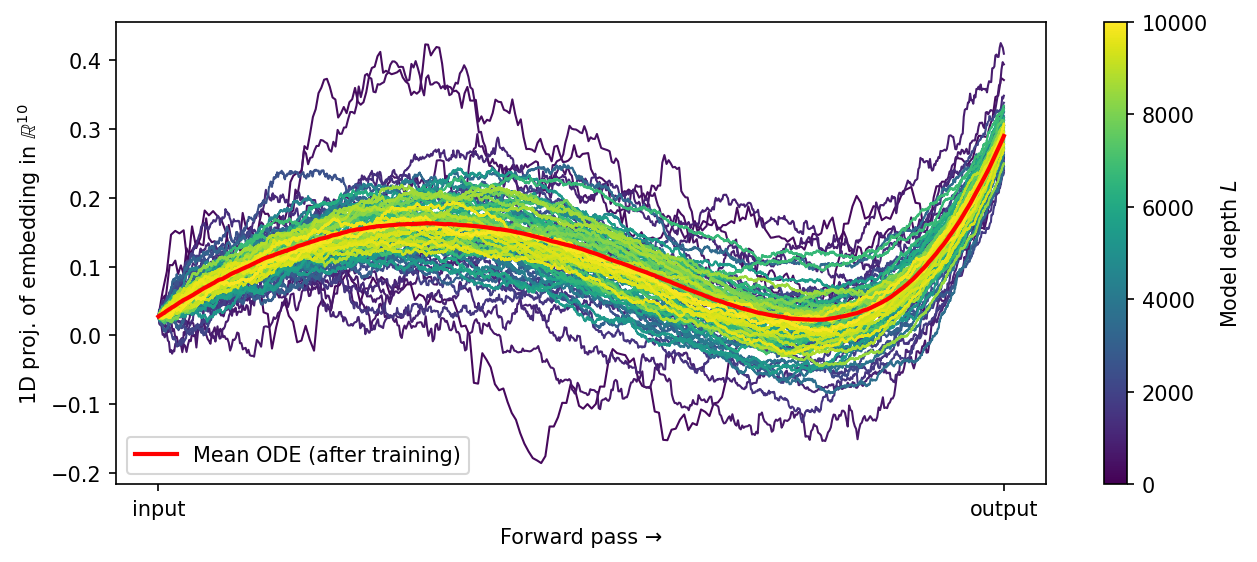}
\caption{Forward pass (1D projection, fixed input) of trained ResNets (after $K=100$ GD iterations) with  two-layer-perceptron blocks, varying depths $L$ and hidden width $M=1$. The red curve shows the corresponding forward pass for the limit model, approximated with a ResNet of very large hidden width and depth (setting detailed in Section~\ref{sec:experiments-LM}). The convergence rate towards the red curve is shown in Figure~\ref{fig:experiment1} and characterized in Theorem~\ref{claim:D-dependence}.}\label{fig:cover}
\end{figure}

\begin{figure}
\centering
\begin{subfigure}{0.48\linewidth}
\centering
\includegraphics[scale=0.5]{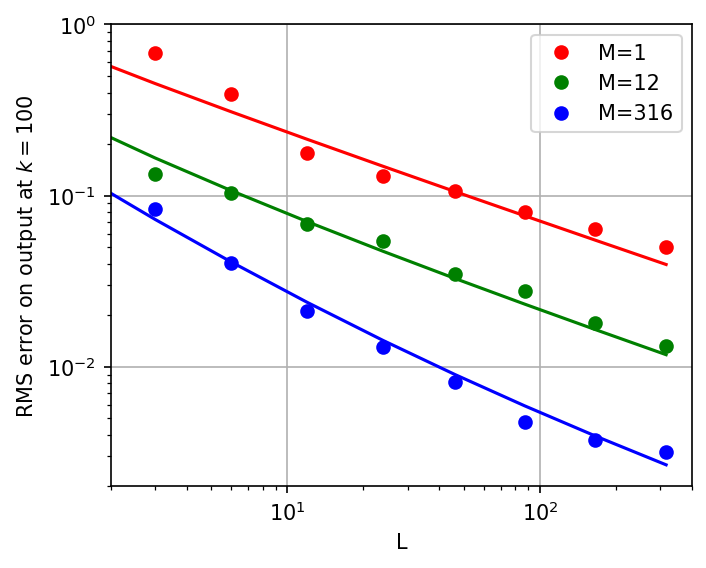}
\caption{Error on output vs depth $L$}
\end{subfigure}%
\begin{subfigure}{0.48\linewidth}
\centering
\includegraphics[scale=0.5]{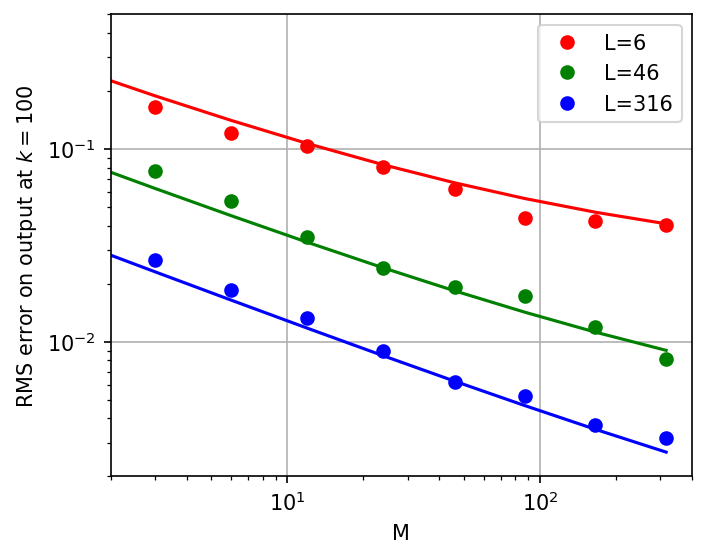}
\caption{Error on output vs hidden width $M$}
\end{subfigure}
\caption{Comparison of the experimental errors (bullets) with the theoretical upper-bound $a/L+b/\sqrt{ML}$ from Theorem~\ref{thm:main} with $a=0.15$ and $b=0.22$ manually adjusted to fit observations (plain lines). The y-axis shows root mean square error (averaged over $10$ random repetitions) on the output after $k=100$ GD steps  (same setting as Figure~\ref{fig:cover}, details in Section~\ref{sec:experiments-LM}).}\label{fig:experiment1}
\end{figure}

    From a mathematical standpoint, our key insights to obtain these estimates are: (i) due to random initialization, the forward and backward passes through a ResNet behave as stochastic approximations of certain mean ODEs, and (ii) by propagation of chaos—i.e., asymptotic independence of the units—this behavior is preserved throughout training. To reflect this interpretation and highlight that the limit does not require $M\to \infty$ (in fact, our viewpoint also applies to well-studied architectures with a single weight matrix per block where $M=1$), we propose to name it the \emph{Neural Mean ODE}, a name inspired by the stochastic approximation literature~\citep{kushner2003stochastic, benaim2006dynamics}.

 \subsection{Summary of contributions}
 The contributions of this paper are broadly divided into two parts: in the first part, we consider generic ResNets architectures and ignore the dependencies in the embedding dimension $D$. In the second part, we focus on ResNets with two-layer perceptrons (2LP) blocks and track the dependencies in $D$.
 
The contributions in the first part, for generic ResNets, can be summarized as follows:
\begin{enumerate}
\item In Theorem~\ref{thm:main}, for a residual scale $\Theta_D\big(\frac{1}{LM}\big)$, we show that after $k$ steps of gradient descent (GD) from a random initialization, the difference between the  the ResNet and the \emph{Neural Mean ODE} is, with high probability, bounded by
\[
O_{D,k}\Big(\frac{1}{L}+\frac{1}{\sqrt{ML}}\Big).
\]
In this case, the dynamics exhibits maximal local feature updates (MLU) and the limit Mean ODE is genuinely non-linearly parameterized.
\item In Theorem~\ref{thm:main-lazy},  for a residual scale $\Theta_D\big(\frac{\alpha}{LM}\big)$ with $1\ll \alpha \ll \sqrt{ML}$, we show that after $k$ steps of GD from a random initialization, the difference between  the ResNet and the \emph{Neural Tangent ODE}, i.e.~the linearization of the Mean ODE's drift around its initial parameters, is with high probability, bounded by
\[
O_{D,k}\Big(\frac{1}{\alpha}+ \frac{1}{L}+\frac{\alpha}{\sqrt{ML}}\Big).
\]
\end{enumerate}
In the second part, we focus on ResNets with two-layer perceptron (2LP) blocks, which are not covered by the assumptions of the first part. We obtain a detailed description of their behavior, including the dependencies in $D$:
\begin{enumerate}[resume]
\item First, we study the phase diagram of this architecture as a function of the residual scale, namely the product of the branch multiplier with the initialization scale of the block's output layer. We identify the residual scale  $O\big(\frac{\sqrt{D}}{LM}\big)$ as necessary and sufficient for having maximal local feature updates (MLU), see the phase diagram in Figure~\ref{fig:sigmav-phase-diagram}. This classification of scalings extends CompleteP~\citep{dey2025don}, known for $M=\Theta(D)$, to general architecture shapes $(L,M,D)$ with non-proportional scalings. 
\item In Theorem~\ref{claim:D-dependence}, our most technical result, we prove that with a residual scale $O\big(\frac{\sqrt{D}}{LM}\big)$ and if $D=O(M)$, the difference between the ResNet (with 2LP blocks and gradient clipping) and its $L\to \infty$ limit is, with high probability, bounded by
\[
O_{k}\Big(\frac{1}{L}+\sqrt{\frac{D}{ML}}\Big).
\]
This confirms the validity of the limit in practical regimes where $M\approx D$ and $ML\gg D$. 
\end{enumerate}

We also verify experimentally\footnote{The code to reproduce the numerical experiments is available at: \url{https://github.com/lchizat/2025-hidden-width-deep-resnet/}} in basic settings that all our predicted rates and phase diagrams are tight in their dependency in $L,M$ and $D$ and the residual scale.

\paragraph{Organization} The contributions for general ResNets are presented in Section~\ref{sec:limit-1} for the MLU regime and Section~\ref{sec:NTODE} for the lazy-ODE regime, and their proofs are together in Section~\ref{sec:proof-generic}. The contributions for ResNets with 2LP blocks are presented in Section~\ref{sec:scaling-D} and their proofs are in Section~\ref{sec:proof-2LP}.

\subsection{Related work}\label{sec:related-work}

\paragraph{Bridging Mean-field and Neural ODE analyses.}  
The first infinite-dimensional analyses of neural network training dynamics appeared in three forms: the Neural ODE framework~\citep{weinan2017proposal, lu2018beyond, chen2018neural}, the mean-field analysis of two-layer perceptrons~\citep{rotskoff2022trainability, chizat2018global, mei2018mean, sirignano2020mean}, and the Neural Tangent Kernel (NTK)~\citep{jacot2018neural, du2019gradient}. Soon after, it was observed that the infinite-depth ($L \to \infty$) and infinite-width ($M \to \infty$) limits could be combined~\citep{lu2020mean, ding2022overparameterization, barboni2024understanding, isobe2023convergence}. These works consider the joint limit $L \to \infty$ and $M \to \infty$, with fixed $D$. In particular, \citep{ding2022overparameterization} obtained convergence of the training dynamics to the limit with an error bound of $O_D\!\left(\frac{1}{L} + \frac{1}{\sqrt{M}}\right)$: the first term corresponds to depth discretization—also present in our analysis—while the second term accounts for fluctuations due to finite width. We also note that their proof technique requires a non-standard initialization with correlations across depth. By adopting a different viewpoint, our analysis improves over this bound and shows that taking $L\to \infty$ from a standard iid initialization, even with $M$ fixed, is sufficient to converge to the same limit.

\paragraph{Large-width HP scalings.}  
The theoretical tractability of the NTK limit stems from an initialization scaling that makes the model asymptotically linear in its parameters. The key role of the initialization scale (or explicit scaling factors) in determining the asymptotic training regime was first emphasized in~\citep{chizat2019lazy}, which also argued that this lazy-kernel\footnote{We write lazy-kernel to mark the distinction with the lazy-ODE regime.} regime is suboptimal due to the absence of feature learning. The classification of HP scalings was then extended to deep MLPs in~\citep{geiger2020disentangling}, and a complete classification for finite depth MLPs was proposed in~\citep{yang2021tensor}. The latter identified \(\mu\)P—combining mean-field scaling in the output layer with standard scaling in other layers—as achieving the maximal (feature) update (MU) regime. It was demonstrated in~\citep{yang2021tuning} that $\mu$P enables zero-shot HP transfer between models of different widths. In our setting, this scaling corresponds to requiring a backward pass with entrywise scale \(1/D\) (this condition appears in our analysis in Section~\ref{sec:scaling-D}).  

\paragraph{Large-depth HP scalings.}   More recently, HP scalings in terms of depth have also been studied~\citep{yang2023tensor, bordelon2023depthwise}, with criteria that singled-out a residual-block scaling of \(\Theta_{M,D}\Big(\frac{1}{\sqrt{L}}\Big)\). However, those works also noticed that this scaling leads to a linearization of each residual block—what we call the lazy ODE regime—and experiments in~\citep{dey2025don} suggest that this behavior might be empirically suboptimal. The mechanism at play in this regime is comparable to the one in the lazy-kernel regime, where the random initial weights over-amplify the effect of the updates of pre-activations in the forward pass, thereby preventing $\Theta(1)$ local feature updates (i.e.~feature updates due to local weight updates, see Section~\ref{sec:scaling-D} for details). They proposed instead  \emph{CompleteP} with residual scale \(\Theta\Big(\frac{1}{L\sqrt{D}}\Big)\) under the assumption \( M = \Theta(D)\).  Relatedly, it was clear from the \emph{Mean-field Neural ODE} literature that the residual scale $\Theta_D\Big(\frac{1}{ML}\Big)$ leads to local feature updates as $M,L\to \infty$ with $D$ fixed. Our analysis allows to bridge these viewpoints and to complete the phase diagram, showing that a necessary and sufficient condition for maximal local feature updates (MLU) when $D=O(LM)$ is the residual scale \(O\big(\frac{\sqrt{D}}{LM}\big)\). 

\paragraph{Other approaches to large neural networks.}  
A variety of other frameworks have been proposed to analyze large neural networks and the role of architectures and HP scalings. Examples include the Neural Network Gaussian Process~\citep{lee2018deep, matthews2018gaussian}, dynamical isometry~\citep{pennington2017resurrecting}, and the study of gradients~\citep{hanin2018neural} or conjugate/tangent kernels at initialization~\citep{hayou2019impact, hayou2021stable}. A limitation of these approaches is that, being restricted to initialization, they do not capture inherently dynamical phenomena such as feature updates, which are critical for identifying optimal scalings. For instance, in the Neural Mean ODE considered here, the first forward and backward passes are asymptotically trivial---they compute the identity map---nevertheless, \citep{dey2025don} found that transformers in this regime can achieve optimal performance in large-scale language modeling tasks. Another line of work concerns the description of the training dynamics via Dynamical Mean Field Theory~\citep{bordelon2022self}, or its algorithmic/programmatic counterpart Tensor Programs~\citep{yang2021tensor}. These works first take the infinite-width limit assuming $M=\Theta(D)$ and then take the large depth limit $L\to \infty$. While this approach has led to the first descriptions of the dynamics of full infinite-size ResNets, the limits are sequential (rather than joint), non-quantitative, limited to the proportional scaling of $M$ and $D$, and rely on heuristic arguments. 

In the companion paper~\cite{chaintron2026resnets}, we study the quantitative $D\to \infty$ limit of the Neural Mean ODE and obtain a $O(1/\sqrt{D})$  high probability bound. Combined with our result in Theorem~\ref{claim:D-dependence}, this leads to a quantitative and rigorous picture of the large-size limits of ResNets with general shapes, see~\cite{chaintron2026resnets} for further discussion.

\section{Generic ResNets in the maximal local update (MLU) regime}\label{sec:limit-1}
In this section, we introduce the training dynamics of ResNets with generic blocks and of the Mean ODE limit model, and then state our quantitative convergence theorem in the MLU regime ($\alpha=\Theta(1)$).

\subsection{Training dynamics of generic ResNets} 
\label{sec:resnet}
Consider a ResNet with depth $L\in \NN^*$, embedding dimension $D\in \NN^*$ and $M\in \NN^*$ units per layers. For an input $x\in \RR^D$, weights $\theta = (z^{j,\ell})_{j,\ell}\in (\RR^p)^{M\times L}$, and scaling factor $\alpha>0$ (think $\alpha=1$ for now), its output $\hat h^L_\theta(x)\in \RR^D$ is computed via the \emph{forward pass} recursion\footnote{ We keep the embedding/unembedding matrices fixed since their behavior is not the focus of this work. One can think of them as being absorbed in the input and the loss.}
\begin{align}\label{eq:discrete-forward}
\hat h^0_\theta(x) = x,&&
\hat h^{\ell}_\theta(x)&=\hat h^{\ell-1}_\theta(x)+ \frac{\alpha}{LM}\sum_{i=1}^M \phi(\hat h^{\ell-1}_\theta(x),z^{i,\ell}),\quad \ell\in [1:L]
\end{align}
where $\phi: \RR^D\times \RR^p\to \RR^D$ represents one ``unit'' parameterized by $z\in \RR^p$, such as a neuron in a vanilla two-layer perceptron, a gated linear unit, an attention head, etc. 

\paragraph{Examples} A ResNet with two-layer perceptron (2LP) blocks without intercepts, is obtained by letting $z=(u,v)\in \RR^{D}\times \RR^D$ (ie $p=2D$) and for $x\in \RR^D$,
\begin{align}\label{eq:MLP-example}
\phi_{\text{mlp}}(x,(u,v))=v\rho(u^\top x/D)
\end{align}
where $\rho:\RR\to\RR$ is the activation function, acting entrywise.  Summing $M$ such units is equivalent to the standard 2LP block $x\mapsto V\rho(D^{-1}Ux)$ with $U\in \RR^{M\times D}$ and $V\in \RR^{D\times M}$ (we introduce here a factor $D^{-1}$ for consistency with the analysis of this architecture in Section~\ref{sec:scaling-D}). 

Also, ResNets architectures with a single weight matrix per block are covered by our analysis by letting $M=1$ and for instance $\phi(x,W)=W\rho(x)$ or $\phi(x,W)=\rho(Wx)$ with $W\in \RR^{D\times D}$ with a centered iid initialization $W_0$. For the latter, observe that if $\rho$ is not odd then $\E[\phi(x,W_0)]\neq 0$, in which case there is no lazy ODE regime (see Section~\ref{sec:NTODE}). 

Another possible block is the attention block, obtained by letting $z=(W_K,W_Q,W_V,W_O)\in (\RR^{d_k\times D})^4$ and for an input family of $T$ tokens $x=(x_1,\dots,x_T)\in (\RR^D)^T$,
\begin{align}\label{eq:attention}
\phi_{\text{att}}(x,z) = \Big(W_O^\top \sum_{i=1}^T \frac{e^{(W_Qx_t)^\top (W_Kx_i)/\sqrt{d_k}}}{\sum_{j=1}^T e^{(W_Qx_t)^\top (W_K x_j)/\sqrt{d_k}}} W_Vx_i\Big)_{0\leq t\leq T} \in (\RR^D)^T.
\end{align}
In this setting, the hidden-width $M$ is known as the number \emph{attention heads} per layer while $d_k$ is the key/query dimension, which is considered a constant in our analysis\footnote{It is in fact not clear whether scaling-up $d_k$ is beneficial. For instance, in the Llama 3.1 family of models, $d_k$ is constant equal to $128$ across all model sizes~\citep{grattafiori2024llama}.}.

\paragraph{Training dynamics} 
Consider a training set of size $n$, where for the $i$-th training sample the input is $x_i\in \RR^D$ and the loss is $\loss_i:\RR^D\to \RR$, assumed differentiable. This leads to an objective function $\hat \Ll$ in the variable  $\theta = (z^{j,\ell})_{j,\ell}\in (\RR^p)^{M\times L}$ defined as:
\begin{align}\label{eq:loss-discrete}
\hat \Ll(\theta) \coloneqq \frac1n \sum_{i=1}^n \hat \Ll_i(\theta), && \hat \Ll_i(\theta)\coloneqq\loss_i(\hat h^L_\theta(x_i)).
\end{align}
Consider an initial probability distribution $\mu_0\in \Pp(\RR^p)$ and a learning-rate $\eta>0$. The gradient descent (GD) dynamics $( \theta_k)_{k\geq 0} = (\hat Z^{i,\ell}_k)_{i,\ell,k}$ is defined by
\begin{align}\label{eq:GD-discrete}
\hat Z^{j,\ell}_0 &\overset{iid}{\sim} \mu_0,&
\hat Z^{j,\ell}_{k+1} &= \hat Z^{j,\ell}_k - \frac{LM\eta}{\alpha^2}\nabla_{z^{j,\ell}} \hat \Ll(\theta_k),\qquad \forall j\in [1:M], \forall \ell\in [1:L], \forall k\in \NN.
\end{align}
We switched to capital letters in the notation to indicate that these quantities are random variables.
As long as $\alpha=\Omega(1)$, the factor $ML/\alpha^2$ is the appropriate LR scaling since it prevents the update of the forward and backward pass from vanishing/exploding asymptotically as is clear from the expression of the gradient (see~\eqref{eq:gradient-discrete} below).
We focus on GD only to fix ideas; our technique would apply to any update rule that is a Lipschitz function of the sample gradients such as GD, SGD, clip SGD, Adam\footnote{For Adam, from~\cite[Eq.(9)]{orvieto2025search}, the Lipschitz property holds uniformly when the sequence of batch gradients has uniformly lower-bounded empirical variance.}, etc.

\paragraph{Expression in terms of update map} 
For $x,w\in \RR^D$ and $\theta\in (\RR^p)^{M\times L}$, 
define the \emph{backward pass} $\hat b^\ell_\theta(x,w) \coloneqq \Big(\frac{\partial \hat h^L_\theta}{\partial \hat h^\ell_\theta}(x)\Big)^\top w \in \RR^D$ where $\frac{\partial \hat h^L_\theta}{\partial \hat h^\ell_\theta}(x) \in \RR^{D\times D}$ is the Jacobian of the map $\hat h^\ell_\theta(x) \mapsto \hat h^L_\theta(x)$ defined by the recursion~\eqref{eq:discrete-forward}. By the chain rule, we have $\forall j\in [1:M]$, $\forall \ell\in [1,L]$,
\begin{align}\label{eq:gradient-discrete}
\nabla_{z^{j,\ell}} \Ll_i(\theta) = \frac{\alpha}{LM}D_2\phi(\hat h^{\ell-1}_\theta(x_i),z^{j,\ell})^\top \hat b^{\ell}_\theta(x_i,\nabla \loss_i(\hat h^{L}_\theta(x_i)))
\end{align}
where $(\hat b^{\ell}_\theta)_{\ell\in [1:L]}$ can be obtained from the backward recursion
\begin{align}\label{eq:discrete-backward}
\hat b^L_\theta(x,w) = w, && \hat b^{\ell-1}_\theta(x,w)=b^{\ell}_\theta(x,w)+\frac{\alpha}{LM} \sum_{j=1}^M D_1\phi(\hat h^{\ell-1}_\theta(x),z^{j,\ell})^\top b^{\ell}_\theta(x,w).
\end{align}
In those expressions, $D_1\phi$ and $D_2\phi$ stand for the Jacobians of $\phi$ in its first and second argument, respectively.
We can therefore rewrite the GD equations defining $(\theta_k)_{k\geq 0}=(\hat Z^{j,\ell}_k)_{j,\ell,k}$ in~\eqref{eq:GD-discrete} as 
\begin{align}\label{eq:discrete-update-simplified}
\hat Z^{j,\ell}_0 &\overset{iid}{\sim} \mu_0,&&
\hat Z^{j,\ell}_{k+1}=\hat Z^{j,\ell}_k + \frac{\eta}{\alpha n} \sum_{i=1}^n \Update(\hat Z^{j,\ell}_k, \hat h^{\ell-1}_{k,i}, \hat b^\ell_{k,i}) \quad \forall k\geq 0
\end{align}
where we have used the shortcuts $\hat h_{k,i}=\hat h_{\theta_k}(x_i)$, $\hat b_{k,i}=\hat b_{\theta_k}(x_i,\hat w_{i,k})$ and $\hat w_{i,k}\coloneqq \nabla \loss_i(\hat h^{L}_{k,i})$ and the (per-sample) update map $\Update: \RR^p\times \RR^D\times \RR^D\to \RR^p$ is defined as
\begin{align}\label{eq:update-map}
\Update(z,h,b) \coloneqq - D_2\phi(h,z)^\top b.
\end{align}

\subsection{Training dynamics of Neural Mean ODEs}\label{sec:NMODE}
We now present the limit model, which we refer to as the (Neural) Mean ODE.  We parameterize this model by a $L^2$ map $Z:[0,1]\times \Omega \to \RR^p$ where $(\Omega,\P)$ is an abstract probability space.  We may interpret $Z$ as a stochastic process indexed by a depth index $s\in [0,1]$ whose distribution given $s$ represents the distribution of parameters at this layer\footnote{Most prior works parameterize the model by the family of probability measures $(\text{Law}(Z(s)))_{s\in [0,1]}$. While for two-layer networks this measure-based representation is appealing — in particular because it convexifies the objective — this advantage disappears for ResNets. In contrast, representing the model as an $L^2$-map (or equivalently as a stochastic process) preserves the natural optimization geometry of finite-depth ResNets without resorting to optimal transport tools, and it allows to capture the evolution of individual parameters. Conceptually, this choice mirrors the classical dichotomy between the PDE and the McKean–Vlasov representations of mean-field interacting particle systems.}.  

The forward pass $h_Z(s,x)\in \RR^D$ is a function of depth $s\in [0,1]$, input $x\in \RR^D$ and the stochastic process $Z$ that encodes the parameters of the limit model. It is characterized as the solution to the forward Mean ODE:
\begin{align}\label{eq:limit-forward}
h_Z(0,x)=x,&& \partial_s h_Z(s,x)= \alpha \E\big[ \phi(h_Z(s,x),Z(s))\big],\quad \forall s\in {[0,1]}, \forall x\in \RR^D.
\end{align}
Note that $h$ depends on $Z$ only via its marginal distributions $(\Law(Z(s)))_{s\in [0,1]}$.
Similarly as in~\eqref{eq:loss-discrete}, the objective is defined as  
\begin{align*}
 \Ll(Z) \coloneqq \frac1n \sum_{i=1}^n \Ll_i(Z), && \Ll_i(Z)\coloneqq\loss_i(h_Z(1,x_i))
\end{align*}
and we consider GD of $\Ll$ in the $L^2$ geometry starting from a random constant $\xi\sim \mu_0$:
\begin{align}\label{eq:GD-continuous}
Z_0(s)=\xi,\; \forall s\in [0,1]&&
Z_{k+1} = Z_{k}  - \frac{\eta}{\alpha^2}\nabla \Ll(Z_k),\quad \forall k\in \NN.
\end{align}
Observe that this is a deterministic dynamics in $L^2([0,1]\times \Omega;\RR^p)$. Our choice to initialize with a random constant function $Z_0(s)=\xi$ is just a convenient convention because only the $s$-marginals of $Z_0$ matter. This convention will allow us to control the regularity in $s$ of the ODE~\eqref{eq:limit-forward} associated to $Z_k$ directly in terms of the regularity of $s\mapsto Z_k(s)$, which is easy to track.

\paragraph{Expression in terms of update map} 
The gradient's expression can be derived with the adjoint method (i.e.~continuous-time backpropagation). The backward Mean ODE $b_Z(s,x,w)\in \RR^D$ with $s\in [0,1]$ and $x,w\in \RR^D$ is the solution to $b(1,x,w,Z)=w$ and
\begin{align}\label{eq:limit-backward}
\partial_s b_Z(s,x,w)= - \alpha \E\Big[ D_1\phi(h_Z(s,x),Z(s))^\top b_Z(s,x,w)\Big] ,\; s\in {[0,1]}.
\end{align}
One then has the following equations for the GD dynamics $(Z_k)_{k\geq 0}$, with $\xi_0\sim \mu_0$ and $\forall s\in [0,1]$
\begin{align}\label{eq:GD-limit}
Z_0(s)=\xi_0, &&
Z_{k+1}(s) = Z_{k}(s) - \frac{\eta}{\alpha n}\sum_{i=1}^n \Update(Z_k(s), h_{k,i}(s), b_{k,i}(s)), && \forall k\geq 0.
\end{align}
where $h_{k,i}(s)\coloneqq h_{Z_k}(s,x_i)$, $b_{k,i}(s)\coloneqq b_{Z_k}(s,x_i,w_{i,k})$, $w_{i,k}\coloneqq \nabla \loss_i(\hat h_{k,i}(1))$ and the per-sample update map $\Update(z,h,b)=-D_2\phi(h,z)^\top b$ is the same as in~\eqref{eq:update-map}.

The rigorous connection between this dynamics and the ResNet dynamics is the object of Theorem~\ref{thm:main} in the next section.

\paragraph{Transformer Mean ODE} Let us mention that our analysis can be easily adapted to deal with various types of blocks in the same ResNet---computed in parallel or sequentially\footnote{The basic argument is that for a scheme defining for $\ell$ even $x_{\ell+1}=x_\ell+L^{-1} f(x_\ell)$ and $x_{\ell+2}=x_{\ell+1}+L^{-1} g(x_{\ell+1})$, the $L\to \infty$ limit is $\dot{x}(s)=f(x(s))/2+g(x(s))/2$ as can be seen from the two-step expansion $x_{\ell+2}=x_{\ell} + 2L^{-1}(f(x_{\ell})/2+g(x_{\ell})/2)+O(L^{-2})$.}. Each block type leads to one term in the Mean ODE. For instance, the Transformer architecture alternates between perceptron~\eqref{eq:MLP-example} and attention blocks~\eqref{eq:attention}. Given a family of tokens $(x_1,\dots,x_T)\in (\RR^{d_{in}})^T$, the Transformer Mean ODE lives in $\RR^{D\times T}$ and is given by
\begin{align*}
\left\{
\begin{aligned}
h(0,x) &=W_{\tin} x\\
\partial_s h(s,x) &= \frac12 \E[\phi_{\text{mlp}}(h(s,x),Z_{\text{mlp}}(s))]+ \frac12 \E[\phi_{\text{att}}(h(s,x),Z_{\text{att}}(s))]\\
f(x) &= W_{\tout}^\top h(1,x)
\end{aligned}
\right.
\end{align*}
where $f(x)$ are the logit outputs, $Z_{\text{att}}:[0,1]\to (\RR^{D\times d_k})^4$ and $Z_{\text{mlp}}:[0,1]\to (\RR^{D})^2$ are stochastic processes that parameterize the limit model and $W_\tin,W_\tout\in \RR^{D\times d_{in}}$ are the embedding and unembedding matrices. Note that in MLP blocks, the tokens are processed independently.

\subsection{Convergence theorem: large-depth limit in the MLU regime}\label{sec:main}
We consider the following regularity assumptions.

\begin{assbox}
\begin{assumption}[Regularity assumptions] \,\label{ass:regularity} There exists $B>0$ such that:
\begin{enumerate}
\item $\phi$ is $B$-Lipschitz, differentiable, its differential $D\phi$ is $B$-Lipschitz and $\Vert \phi(0,0)\Vert_2 \leq B$;
\item The losses $\mathrm{loss}_i$ are differentiable with $B$-Lipschitz derivatives and $\Vert \nabla \mathrm{loss}_i(0)\Vert_2\leq B$;
\item The inputs satisfy $\max_i \Vert x_i\Vert_2\leq B$.
\end{enumerate}
\end{assumption}
\end{assbox}
The regularity that we assume on $\phi$ in this section is quite restrictive but allows us to focus on the main mechanisms in our proofs. In Section~\ref{sec:scaling-D}, we carry a similar analysis but in the case of 2LP blocks where $\phi$ and $D\phi$ are only pseudo-Lipschitz (i.e. locally Lipschitz with a controlled growth).

We recall that a $\RR^p$-valued random variable $Z$ is said subgaussian with variance-proxy $\sigma^2>0$ (written $\sigma^2$-subgaussian) if for all $u\in \RR^p$ with $\Vert u\Vert_2=1$  and $\lambda\in \RR$, it holds
\begin{align}\label{eq:variance-proxy}
\E [\exp(\lambda u^\top (Z-\E Z))] \leq e^{\sigma^2\lambda^2/2}.
\end{align}
We say that a probability measure $\mu\in \Pp(\RR^p)$ is $\sigma^2$-subgaussian if $Z$ is  $\sigma^2$-subgaussian for any (and therefore all) $Z\sim\mu$. Background and useful results on subgaussian random variables are gathered in Appendix~\ref{sec:subgaussian}.

In order to state our first convergence theorem, consider the ResNet's dynamics $(\hat Z^{j,\ell}_0)$ (defined in~\eqref{eq:discrete-update-simplified}) and consider $M\times L$ independent copies of the limit dynamics $(Z^{j,\ell}_k)_{k\geq 0}$ (defined in~\eqref{eq:GD-limit}) coupled via $Z^{j,\ell}_0(s)=\hat Z^{j,\ell}_0$, $\forall s\in [0,1]$. Consider the distance between the two dynamics in parameter space, forward pass and backward pass respectively defined, with $s_\ell=\ell/L$ for $\ell\in [0:L]$, as
\begin{align*}
\Delta_k^Z &\coloneqq \max_{\substack{j\in [1:M] \\ \ell\in [1:L]}} \Vert \hat Z^{j,\ell}_k - Z^{j,\ell}_k(s_{\ell-1})\Vert,&
\Delta_k^h &\coloneqq \max_{\substack{i\in [1:n] \\ \ell\in [0:L]}} \Vert \hat h^\ell_{k,i} - h_{k,i}(s_\ell)\Vert,&
\Delta_k^b &\coloneqq \max_{\substack{i\in [1:n] \\ \ell\in [0:L]}} \Vert \hat b_{k,i}^\ell - b_{k,i}(s_\ell)\Vert.
\end{align*}
In this section, since we do not track the dimensional constants, the norm $\Vert \cdot \Vert$ is arbitrary.

\begin{thmbox}
\begin{theorem}[Convergence in the MLU regime]\label{thm:main}
Let Assumption~\ref{ass:regularity} hold with $B>0$, let $\alpha=1$ (MLU regime), and let $\mu_0\in \Pp(\RR^p)$ be a subgaussian distribution with variance proxy $\sigma_0^2\leq B$. 
Then $\forall k\geq 0$, there exists $c_1,c_2>0$  that only depend on $B,D$ and $k\eta$ such that with probability at least $1-\delta$, it holds:
 \begin{align}\label{eq:main-thm-convergence}
\max_{k'\in [0:k]} \big\{ \Delta_{k'}^Z,\Delta_{k'}^h, \Delta_{k'}^b\big\} \leq c_1 \left( \frac{1}{L}+ \frac{\sigma_0(1+\sqrt{\log(kn/\delta)})}{\sqrt{LM}}\right)
\end{align}
provided that the right-hand side is smaller than $c_2$.
\end{theorem}
\end{thmbox}

We can make the following comments:
\begin{itemize}
\item These errors bounds are the sum of a depth-discretization error in $O(1/L)$, and a sampling error in $O(\sigma_0/\sqrt{LM})$. Notably, the latter only depends on the product $LM$ which can therefore be interpreted as an \emph{effective width} of the architecture. We experimentally confirm in Figure~\ref{fig:experiment1} that these rates are tight in their dependency in $L$ and $M$. Observe that $L\to \infty$ is sufficient for the bound to vanish.
\item The case $\sigma_0=0$ corresponds to a deterministic initialization and there is no sampling error in this case. This is the classical Neural ODE setting, studied in~\citep{avelin2021neural, marion2023implicit}.
\item  Note that the convergence of the parameters $\hat Z^{j,\ell}_k$ is towards a stochastic object while the convergence of the forward and backward pass is towards deterministic limits.
\end{itemize}

\paragraph{Proof idea: stochastic approximation and propagation of chaos} Let us briefly explain the proof idea of Theorem~\ref{thm:main}. To start, suppose that we would like to compute numerically the forward Mean ODE~\eqref{eq:limit-forward} at iteration $k$ of GD for a given input $x$:
\begin{align}\label{eq:proof-idea-0}
h_k(0,x)&=x,& \partial_s h_k(s,x) = \E[\phi(h_k(s,x),Z_k(s))],\; s\in [0,1].
\end{align}
A natural scheme in this context is the ``Euler/Monte-Carlo'' scheme 
\begin{align}\label{eq:proof-idea}
\bar h^0_k(x)=x,&&
\bar h^{\ell}_k(x) = \bar h^{\ell-1}_k(x)+\frac{1}{ML} \sum_{j=1}^M \phi(\bar h^{\ell-1}_k(x), Z^{j,\ell}_k(s_{\ell-1})), \; \forall \ell\in [1:L]
\end{align}
where $Z^{j,\ell}_k$ are $M\times L$ independent samples from $Z_k$. By classical arguments from stochastic approximation (see Lemma~\ref{lem:SA} in our case), it can be shown that this scheme has, with high probability, an approximation error of order $O\big(\frac{\mathrm{Lip}_k}{L} + \frac{\sigma_k}{\sqrt{ML}}\big)$, that is the sum of the Euler scheme error with step $1/L$ and the Monte-Carlo error with $M\times L$ samples. Here $\Lip_k$ is the Lipschitz constant of $s\mapsto Z_k(s)$ and $\sigma_k$ is a variance-proxy of the random variable $\phi(h_k(s,x),Z_k(s))$ (uniformly in $s$). In the context of Theorem~\ref{thm:main}, it is not difficult to obtain estimates for these two quantities, see Lemmas~\ref{lem:propagation-regularity} and~\ref{lem:propagation-subgaussianity}.

Clearly, the forward pass after $k$ steps of GD in a ResNet is very similar to~\eqref{eq:proof-idea}, but there is an important difference: in the ResNet, the $\hat Z_k^{j,\ell}$ are not sampled from the limit dynamics and are not independent, except at $k=0$. The core of the proof of Theorem~\ref{thm:main} consists then in an argument by recursion over $k$ to jointly control $\Delta_k^Z$, $\Delta_k^h$ and $\Delta_k^b$. This argument shows that the only new source of error at each iteration if the approximation error of~\eqref{eq:proof-idea}, and its analog for the backward pass. This proof scheme is common in the  \emph{propagation of chaos} literature~\citep{dobrushin1979vlasov, sznitman2006topics}, although the fact that here the particles---the $M\times L$ samples $(Z^{j,\ell})$ of the stochastic process---interact through a system of stochastically approximated ODEs adds a layer of complexity.

\begin{remark}[Analogy between ResNets and SGD]
There is a direct analogy between the convergence of~\eqref{eq:proof-idea} to the Mean ODE~\eqref{eq:proof-idea-0} and the classical result that mini-batch SGD converges to gradient flow as the LR tends to zero. In this analogy, $1/L$ plays the role of the learning rate and $M$ corresponds to the batch size. For SGD with a fixed compute budget $M\times L$, increasing the batch-size $M$ does not accelerate convergence towards gradient flow but enables greater parallelism. Our analysis shows that the trade-off between $M$ and $L$ (for $L\times M$ fixed) in a ResNet follows precisely the same principle.
\end{remark}

\subsection{Experimental validation}\label{sec:experiments-LM} 
We plot on Figure~\ref{fig:experiment1} the distance between the output and its limit as a function of $L$ and $M$ and compare it with the rate $a/L+b/\sqrt{ML}$ with adjusted coefficients $(a,b)$. We observe a very good agreement even though the distance is measured after $k=100$ GD iterations, where the ResNet is close to the end of training (for our specific toy setting), as confirmed by the loss plot on Figure~\ref{fig:loss-log}. Figure~\ref{fig:lipschitz} shows evidence that the dynamics is in the maximal local updates (MLU) regime since the displacement of the parameters is in $\Theta(1)$  (as a matter of fact, since 2LP blocks are linear in $v$, what distinguishes the MLU regime is the displacement of the \emph{input weights} $u$ of each block which is also $\Theta(1)$ in this case, see Section~\ref{sec:scaling-D}). Figure~\ref{fig:lipschitz} also illustrates the regularity of $(s,k) \mapsto Z_k(s)$ proved later in Lemma~\ref{lem:propagation-regularity}-(iii) and which plays a key role in our theory.

\paragraph{Experimental setting.} The training data is $n=10$ input/output pairs with $\Nn(0,1)$ iid entries in embedding dimension $D=10$, the objective is the mean square loss and the residual blocks are two-layer perceptrons~\eqref{eq:MLP-example} with $\rho=\tanh$ nonlinearity. All weights initialized with $\Nn(0,\sqrt{D})$ entries and the LRs are $(\eta_u,\eta_v)=(D,D)$ in accordance with the prescriptions of Section~\ref{sec:scaling-D}. We use  a large ResNet with $M=L=10^3$ as a proxy for the Mean ODE dynamics.

\begin{figure}[h!]
\centering
\begin{subfigure}{0.35\linewidth}
\centering
\includegraphics[scale=0.4]{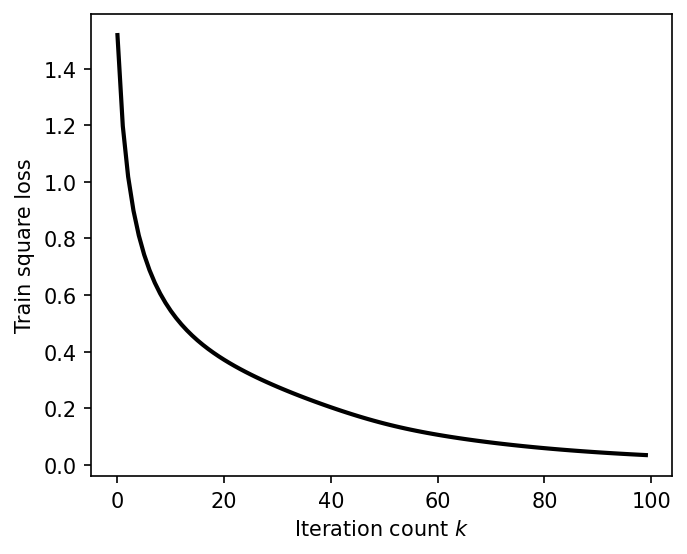}
\caption{Evolution of the square loss}\label{fig:loss-log}
\end{subfigure}%
\begin{subfigure}{0.65\linewidth}
\centering
\includegraphics[scale=0.48,trim= 0 0 5 0,clip]{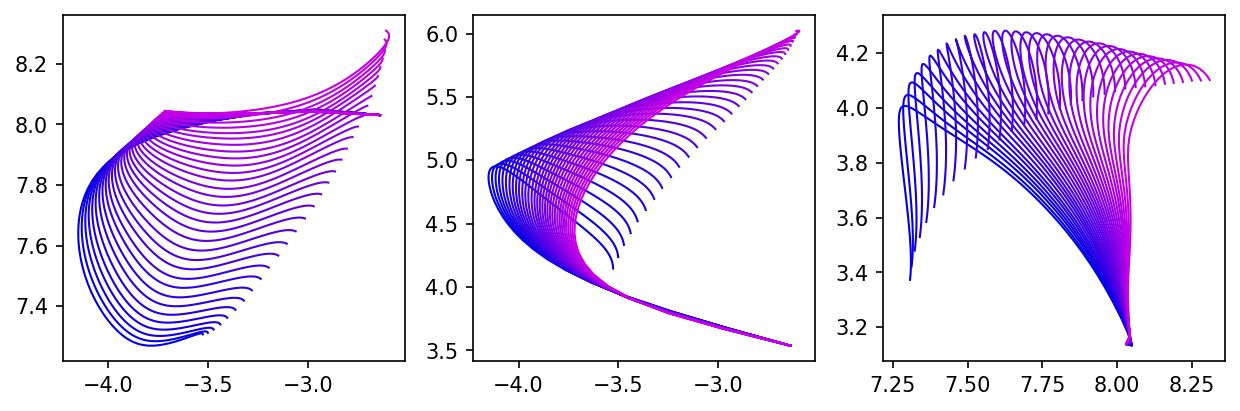}
\caption{Evolution of the weights $(u^{1,\ell}_k)$}\label{fig:lipschitz}
\end{subfigure}%
\caption{(left) The square loss of the Mean ODE model is close to $0$ at $k=100$ indicating approximate convergence (right) Various 2D projections of the curve in $\RR^D$ representing the evolution of the weight $(\hat U^{1,\ell}_k)_{k\in [1:100]}$ where $\ell$ ranges from $1$ (blue) to $L$ (purple). For the purpose of illustration and for this plot only, we have initialized $(\hat U^{1,\ell}_0,\hat V^{1,\ell}_0)=(U_0,V_0)$ $\forall \ell$ (while the rest of the weights for $j\geq 2$ are independently initialized across layers). This illustrates two important properties: (i) the evolution of $\hat U$ is significant (MLU regime) and (ii) the map $(\ell,k) \mapsto U_k(s_{\ell-1}) \approx \hat U^{1,\ell}_k$ is regular in $\ell$ and $k$.}\label{fig:experiment2}
\end{figure}

\section{Generic ResNets in the lazy-ODE regime}\label{sec:NTODE}
When $\alpha\to \infty$, the limit model is different and corresponds to a linearization of the Mean ODE model~\eqref{eq:limit-forward} around $Z\approx Z_0$. In this section, to ensure that the initial forward and backward passes do not explode as $\alpha\to \infty$, we assume that the initialization $Z_0$ and the nonlinearity $\phi$ are such that $\E[\phi(x,Z_0)]=0$ and $\E[ D_1\phi(x,Z_0)]=0$, $\forall x\in \RR^D$.

\subsection{Training dynamics of Tangent Mean ODEs}
We parameterize the limit model by a random pair $(Z_0,\zeta)$ where $\zeta: [0,1]\to \RR^p$ represent the updates of the parameters and $Z_0\in \RR^p$ represents the initialization. At an informal level, $\zeta$ is related to the parameterization $Z$ of the previous section via $\zeta = \alpha\cdot(Z-Z_0)$.

At first order in $\alpha^{-1}$, we have $$\alpha \phi(x,Z(s))=\alpha \phi(x,Z_0+\alpha^{-1} \zeta(s))\approx \alpha \phi(x,Z_0(s))+D_2 \phi(x,Z_0)\zeta(s)+O(\alpha^{-1}).$$ After taking expectation, the first term vanishes by assumption (however, it is important to keep in mind that this term contributes to non-asymptotic fluctuations, all the more that they are amplified by the factor $\alpha$). This suggests to define, in this regime, the forward pass $\underline h_\zeta(s,x) \in \RR^D$ (with an implicit dependency in the law of $Z_0$) as the solution to the (forward Neural) \emph{tangent ODE}
\begin{align}\label{eq:forward-tangent-ODE}
\underline h_\zeta(0,x) = x,&& \partial_s \underline h_\zeta(s,x)=\E[ D_2\phi(\underline h_\zeta(s,x),Z_0) \zeta(s)].
\end{align}
Although the right-hand side of this ODE is linear in the parameter $\zeta$, the output $\underline h_\zeta(1,x)$ remains nonlinear both in $x$ and in $\zeta$. Analogously, the backward tangent ODE is the solution to $\underline b_\zeta(1,x,w)=w$ and
\begin{align}\label{eq:backward-tangent-ODE}
 \partial_s \underline b_\zeta(s,x,w)= -  \E\Big[ D_{1,2}\phi(\underline h_\zeta(s,x),Z_0)^{*_1} [\underline b_\zeta(s,x,w),\zeta(s)]\Big] ,\; s\in {[0,1]}
\end{align}
where $D_{1,2}\phi(x,z)^{*_1}$ is the partial adjoint in the first variable of the mixed second derivative of $\phi$, which we interpret as a linear operator with type $\RR^D\times \RR^p\to \RR^D$.
The equations driving the training dynamics $(\zeta_k)_{k\geq 0}$, which can again be interpreted as a GD in the $L^2$ geometry, initialized at $0$ are, with $\xi_0\sim \mu_0$
\begin{align}\label{eq:GD-limit-lazy}
Z_0(s)\sim \xi_0,&&
\zeta_0(s)=0,&&
\zeta_{k+1}(s) = \zeta_{k}(s) - \frac{\eta}{n} \sum_{i=1}^n \Update(Z_0(s),\underline h_{k,i}(s), \underline b_{k,i}(s))
\end{align}
where we used the abbreviations $\underline h_{k,i}(s) \coloneqq \underline h_{\zeta_k}(s,x_i)$, $\underline b_{k,i}(s) \coloneqq \underline b_{\zeta_k}(s,x_i,\underline w_{i,k})$, $\underline w_{i,k}\coloneqq \nabla \loss_i(\underline h_{k,i}(1)$ and the map $\Update$ is defined in~\eqref{eq:update-map}.
Observe that in~\eqref{eq:GD-limit-lazy}, the update map is evaluated at $Z_0$ (irrespective of the value of $\zeta_k$) while in~\eqref{eq:GD-limit}, it is evaluated at $Z_k(s)$. This is because the value of $Z_k(s) \approx Z_0+\alpha^{-1}\zeta_k(s)$ approaches $Z_0$ in the $\alpha\to +\infty$ limit.

\subsection{Convergence theorem: large-depth limit in the lazy-ODE regime}

To handle the $\alpha\to \infty$ limit and this linearized model, we require one more degree of regularity on $\phi$ and we require that the initial forward and backward passes are centered.
\begin{assbox}
\begin{assumption}[Lazy-ODE regime regularity assumptions] \,\label{ass:lazy-regularity} Assumption~\ref{ass:regularity} holds with $B>0$ and moreover:
\begin{enumerate}
\item $\phi$ is twice differentiable with a $B$-Lipschitz cross differential $D_{1,2}\phi$;
\item it holds 
$\E[ \phi(x,Z_0)]=\E[ D_1\phi(x,Z_0)]=0$, $\forall x\in \RR^D$.
\end{enumerate}
\end{assumption}
\end{assbox}

We are now ready to state the convergence theorem in the $\alpha\to \infty$ case. As before, consider the ResNet's dynamics $(\hat Z^{j,\ell}_0)$ (defined in~\eqref{eq:discrete-update-simplified}) and consider $M\times L$ independent copies of the limit dynamics $(Z_0^{j,\ell},\zeta^{j,\ell}_k)_{k\geq 0}$ (defined in~\eqref{eq:GD-limit-lazy}) such that $Z^{j,\ell}_0(s)=\hat Z^{j,\ell}_0$, $\forall s\in [0,1]$. Consider the distances between the two dynamics  defined, with $s_\ell=\ell/L$ for $\ell\in [0:L]$, as
\begin{align*}
\Delta_k^\zeta &\coloneqq \max_{\substack{j\in [1:M] \\ \ell\in [1:L]}} \Vert \alpha(\hat Z_k^{j,\ell}-\hat Z_0^{j,\ell}) -\zeta_k^{j,\ell}(s_{\ell-1}) \Vert,\\
\Delta_k^{\underline h} &\coloneqq \max_{\substack{i\in [1:n] \\ \ell\in [0:L]}} \Vert \hat h_{k,i}^\ell - \underline h_{k,i}(s_\ell)\Vert,\\
\Delta_k^{\underline b} &\coloneqq \max_{\substack{i\in [1:n] \\ \ell\in [0:L]}} \Vert \hat b_{k,i}^\ell(x_i) - \underline b_{k,i}(s_\ell)\Vert.
\end{align*}

\begin{thmbox}
\begin{theorem}[Convergence in the lazy ODE regime]\label{thm:main-lazy}
Let Assumption~\ref{ass:lazy-regularity} hold with $B>0$, let $\alpha\geq 1$ and let $\mu_0\in \Pp(\RR^p)$ be a subgaussian distribution with variance proxy $\sigma_0^2\leq B$. 
Then $\forall k\geq 0$, there exist $c_1,c_2>0$ that only depend on $B,D$ and $k\eta$ such that with probability at least $1-\delta$, it holds:
\begin{align}
\max_{k'\leq k}\big\{\Delta_{k'}^\zeta, \Delta_{k'}^{\underline h}, \Delta_{k'}^{\underline b} \big\} \leq c_1 \left( \frac{1}{\alpha}+ \frac{1}{L}+\frac{\alpha \sigma_0 (1+\sqrt{\log(kn/\delta)})}{\sqrt{LM}}\right)
\end{align}
provided the right-hand side is smaller than $c_2$.
\end{theorem}
\end{thmbox}

 In the lazy ODE regime $\alpha\to \infty$ with a fixed initialization scale $\sigma_0$, the bound vanishes if and only if $1\ll \alpha\ll \sqrt{LM}$. For $\alpha =\Theta(\sqrt{LM})$, we still expect a similar linearization behavior however the limit is different because the random fluctuations at initialization do not vanish anymore—in particular the first forward pass is described by an SDE (see e.g.~\citep{yang2023tensor, bordelon2023depthwise}). Observe that in parameter space the updates are in $\Theta(1/\alpha)$ while they are in $\Theta(1)$ in the forward pass. This phenomenon is similar to what happens in the lazy-kernel regime~\citep{chizat2019lazy} (on a side note, in the lazy-kernel regime, the output is linear in the parameters, so the lazy-kernel regime implies the lazy-ODE regime but the converse is not true).

\section{Two-layer perceptron blocks and explicit scalings in $D$}\label{sec:scaling-D}
In this section, we extend our results to take into account the dependency in the embedding dimension $D$; both in the asymptotic behavior and in the error bounds. 
For the sake of concreteness, we limit ourselves to the particular case of ResNets with two-layer perceptron (2LP) residual blocks (which was not covered by the generic results of the previous section due to a lack of global Lipschitzness).

\paragraph{Notation for RMS.} 
In the large $D$ setting, it is convenient to manipulate the root-mean-square (RMS) norm rather than the $\ell_2$ norm.  For vectors $x,y\in \RR^D$, the RMS dot product is defined as $\ipD{x}{y} = (x^\top y)/D$ and the RMS norm as $\Vert x\Vert_\RMS \coloneqq \sqrt{\ipD{x}{x}}=D^{-1/2}\Vert x\Vert_2$. The RMS norm can be interpreted as the typical entrysize of $x$ when $x$ is not sparse.  Throughout, whenever we refer to the \emph{scale} of a vector, we mean its RMS norm.

\subsection{Training dynamics of ResNets with 2LP blocks}\label{sec:2LP-setup} 
We consider this section the following architecture, parameterized by $\theta=((u^{j,\ell},v^{j,\ell})_{j,\ell})\in (\RR^D\times \RR^D)^{L\times M}$,
\begin{align}\label{eq:parameterization-uv}
\left\{
\begin{aligned}
\hat h^{0}_\theta(x) &= W_\tin x, \\
\hat h^{\ell}_\theta(x) &= \hat h^{\ell-1}_\theta(x) +\frac{1}{LM} \sum_{j=1}^M v^{j,\ell} \rho\big(\ipD{u^{j,\ell}}{ \hat h^{\ell-1}_\theta(x)} \big),\quad \ell\in [1:L] \\ \hat f_\theta(x)&= \frac1D W_\tout^\top \hat h^L_\theta(x)
\end{aligned}
\right.
\end{align}
where $\rho:\RR\to\RR$ is a smooth nonlinearity. Here the input is $x\in \RR^{d_\tin}$, the output is $\hat f_\theta(x)\in \RR^{d_\tout}$. We have also introduced the embedding $W_\tin\in \RR^{D\times d_\tin}$ and unembedding $W_\tout\in \RR^{D\times d_\tout}$ matrices. In order to keep notations light, we consider these matrices fixed; but it would be immediate to extend our analysis to the case where these matrices are trained. 

Let $\mu_0=(\mu^u_0)^{\otimes D}\otimes (\mu^v_0)^{\otimes D}$ where $\mu^u_0,\mu^v_0\in \Pp(\RR)$ have mean $0$ and entrywise variance $\sigma_u^2$ and $\sigma_v^2$ respectively.
We consider $(\theta_k)_k=((\hat U^{j,\ell},\hat V^{j,\ell})_{j,\ell})_k$ the iterates of GD on the loss $\hat \Ll$ defined as in~\eqref{eq:loss-discrete} with $\hat \Ll_i(\theta)=\loss_i(\hat f_\theta(x_i))$ from a random initialization and LRs $(\eta_u,\eta_v)$:
\begin{align}\label{eq:2LP-dynamics}
\left\{
\begin{aligned}
\hat U^{j,\ell}_{0} \overset{iid}{\sim}  (\mu_0^u)^{\otimes D}\\
\hat V^{j,\ell}_{0} \overset{iid}{\sim}  (\mu_0^v)^{\otimes D}
\end{aligned}
\right. ,&&
\left\{
\begin{aligned}
\hat U^{j,\ell}_{k+1} &= \hat U^{j,\ell}_k -  \eta_{u} LMD\nabla_{u^{j,\ell}} \hat \Ll(\theta_k)\\ \hat V^{j,\ell}_{k+1} &= \hat V^{j,\ell}_k - \eta_{v} LMD  \nabla_{v^{j,\ell}} \hat \Ll(\theta_k)
\end{aligned}
\right.
.
\end{align}
 Note also that we have already pre-multiplied the LR by $LMD$ (this is consistent with the previous sections where the factor was $LM$ and $D$ was $\Theta(1)$). 

 Consider $\hat b^{\ell}_{k,i} \coloneqq {\color{red}D} \big(\frac{\partial \Ll_i}{\partial h^{\ell}}\big)^\top \nabla \loss_i(\hat f_{\theta_k}(x_i))\in \RR^D$ the \emph{normalized} backward pass  at iteration $k$ and sample $i$ (we insist on the factor $D$, introduced to obtain a $\Theta(1)$ RMS norm). It is given at GD iteration $k$ by the backward recursion
 \begin{align}
 \hat b_{k,i}^L &= -W_\tout \nabla \loss_i(\hat f_{k,i}),& \hat b_{k,i}^{\ell-1} = \hat b_{k,i}^\ell + \frac{1}{LM} \sum_{j=1}^M \rho'(\hat P_{k,i}^{j,\ell}) \cdot \hat Q_{k,i}^{j,\ell}\cdot \hat U^{j,\ell}_k
 \end{align}
 where $\hat P_{k,i}^{j,\ell}\coloneqq \ipD{ \hat U_k^{j,\ell}}{\hat h_{k,i}^{\ell-1}}$ are the preactivations and $\hat Q_{k,i}^{j,\ell}\coloneqq \ipD{ \hat V_k^{j,\ell}}{\hat b_{k,i}^{\ell}}$. In these notations, the update equations can be written more explicitly as:
\begin{align}\label{eq:2LP-update-finite}
\left\{
\begin{aligned}
\hat U^{j,\ell}_{k+1} &= \hat U^{j,\ell}_k -  \eta_{u} \sum_{i=1}^n \rho'\big(\hat P_{k,i}^{j,\ell}\big) \cdot \hat Q_{k,i}^{j,\ell} \cdot \hat h^{\ell-1}_{k,i}\\ \hat V^{j,\ell}_{k+1} &= \hat V^{j,\ell}_k - \eta_{v} \sum_{i=1}^n  \rho\big(\hat P_{k,i}^{j,\ell}\big) \cdot \hat b^{\ell}_{k,i}
\end{aligned}
\right.
.
\end{align}

\subsection{Derivation of the large $L,M,D$ phase diagram} \label{sec:2LP-phase-diagram}
We now discuss the effect of the scalings of the four HPs $\eta_u,\eta_v,\sigma_u,\sigma_v$ as a function of the scaling of $L,M$ \emph{and $D$}. In this section, we proceed informally and do not justify all our claims for the sake of conciseness (from the next section onward we proceed again at a rigorous level).

The asymptotic phases can be described simply by comparing the scales of \textbf{loss decay} and  \textbf{local feature updates} (the following computations are in the spirit of~\cite{chizat2024feature}). We consider for convenience the following continuous-time dynamics with a single input sample (in new, but transparent notations):
\begin{align}
    \partial_t \hat U^{j,\ell}_t &=-\eta_u \rho'(\hat P_t^{j,\ell}) \cdot \hat Q_t^{j,\ell} \cdot \hat h_t^{\ell-1}&
    \partial_t \hat V^{j,\ell}_t &=-\eta_v \rho(\hat P_t^{j,\ell}) \cdot b_t^\ell\\
    \hat P_t^{j,\ell} &\coloneqq \ipD{ \hat U^{j,\ell}_t}{\hat h^{\ell-1}_t}&
    \hat Q_t^{j,\ell} &\coloneqq \ipD{ \hat V^{j,\ell}_t}{\hat b^{\ell}_t}.
\end{align}
Let us right away fix some HPs using standard criteria:
\begin{enumerate}
\item We require $ \sigma_\tin \coloneqq \Vert W_\tin x\Vert_\RMS = \Theta(1)$, $\sigma_u=\Theta(\sqrt{D})$ and $\sigma_v=O(\sqrt{LM})$. This ensures proper signal propagation at initialization $\E^{1/2} \Vert \hat h^{\ell}_0\Vert_\RMS^2=\Theta(1)$ for $\ell\in [0:L]$ and feature\footnote{Let us recall that for ResNets with 2LP blocks, the \emph{features} of input $x$ are the $L\times M$ scalars $\rho(\hat P^{j,\ell}(x))$, and not the vectors $\hat h^\ell(x)$ of length $D$, which are linear projections of the features.} diversity at initialization $\E^{1/2}\vert \hat P^{j,\ell}_0\vert^2 =\Theta(1)$ for $\ell\in [1:L]$, $j\in [1:M]$.
\item  We moreover require $\sigma_\tout \coloneqq \Vert W_\tout \nabla \loss(\cdot)\Vert_\RMS  =\Theta(1)$, which implies that the \emph{normalized} backward pass satisfies $\E^{1/2} [\Vert \hat b^\ell_0\Vert^2_\RMS]=\Theta(1)$. This ensures that, for a loss decay in $\Theta(1)$, the forward pass evolves in $\Theta(1)$. This is the key property distinguishing the mean-field/maximal-update regime from the lazy-kernel regime (see e.g.~\citep[Prop.~4.1]{chizat2024feature}).
\end{enumerate}

Now, the evolution of the preactivations can be decomposed as
\begin{align}
\partial_t \hat P_t^{j,\ell} &= \underbrace{\ipD{\partial_t \hat U_t^{j,\ell}}{\hat h_t^{\ell-1}}}_{\text{local update}} + \underbrace{\ipD{ \hat U_t^{j,\ell}}{ \partial_t \hat h_t^{\ell-1}}}_{\text{global update}}
\end{align}
and in particular the local update, denoted $\delta_t \hat P_t^{j,\ell}$, is given by
\begin{align}\label{eq:LFU}
\delta_t \hat P_t^{j,\ell} \coloneqq \ipD{\partial_t \hat U_t^{j,\ell}}{\hat h_t^{\ell-1}} = -\eta_u \rho'(\hat P_t^{j,\ell}) \cdot \hat Q_t^{j,\ell} \cdot  \Vert \hat h_t^{\ell-1}\Vert_\RMS^2 .
\end{align}
For small training times $t$ we have $\hat Q_t^{j,\ell}=\ipD{\hat V_0^{j,\ell}}{\hat b_t^{\ell}} + t \eta_v \rho(\hat P_0^{j,\ell}) \Vert \hat b^\ell_0\Vert^2_\RMS +o(t)$ hence $\E^{1/2}[|\hat Q_t^{j,\ell}|^2]=\Theta \Big(\sqrt{\sigma_v^2/D+t^2\eta^2_v}\Big)$ (using the independence of $\hat V_0^{j,\ell}$ and $\hat b_0^\ell$) so the scale of the local feature update is:
\begin{align}\label{eq:scale-LFU}
\E^{1/2}[|\delta_t \hat P_t^{j,\ell}|^2] = \Theta\Big(\eta_u\sqrt{\frac{\sigma_v^2}{D}+t^2\eta^2_v}\Big).
\end{align}
On the other hand, the evolution of the loss is, by the chain rule,
\begin{align*}
\partial_t \hat \Ll(\theta_t) &= -\frac{\eta_u}{ML} \sum_{j,\ell} \rho'(\hat P_t^{j,\ell})^2 \cdot (\hat Q_t^{j,\ell})^2 \cdot  \Vert \hat h_t^{\ell-1}\Vert_\RMS^2 - \frac{\eta_v}{ML}\sum_{j,\ell}  \rho(\hat P_t^{j,\ell})^2 \cdot  \Vert  \hat b_t^{\ell}\Vert_\RMS^2\\
&=\Theta\Big(\eta_u \Big(\frac{\sigma_v^2}{D} +t^2\eta_v^2\Big) + \eta_v \Big)
\end{align*}
where the last estimate is with high probability. Then,  having a loss decay in $O(1)$ over bounded training time implies
\[
|\partial_t \hat \Ll(\theta_0)|=O(1) \implies 
\Big[\eta_u=O(D/\sigma_v^2) \text{ and } \eta_v = O(1) \text{ and } \eta_u\eta_v^2 = O(1) \Big].
\]
We call such LRs ``stable''. Clearly, the largest stable LRs with balanced contributions are given by $\eta_v=\Theta(1)$ and $\eta_u=\Theta(\min\{1, D/\sigma_v^2\})$.
Therefore, we have the following behaviors for stable LRs: 
\begin{itemize}
    \item (Lazy ODE regime) If $\sigma_v \gg \sqrt{D}$, then $\eta_u=\Theta(D/\sigma^2_v)$ and by~\eqref{eq:scale-LFU}, the scale of the local updates is vanishing at a rate $\Theta(\sqrt{D}/\sigma_v)$.
    \item (Maximal Local Update (MLU) regime) If instead $\sigma_v=O(\sqrt{D})$, then the maximal stable LRs are $\eta_u,\eta_v=\Theta(1)$ and we have loss decay and local feature updates in $\Theta(1)$.
\end{itemize}

The phase diagram obtained from this discussion is represented on Figure~\ref{fig:sigmav-phase-diagram}. Interestingly, this phase diagram and its derivation, share a lot with the phase diagram and derivation of a 2LP of hidden width $M\times L$ and diverging input and output dimension $D$. In the rest of the paper, we focus on a rigorous analysis in the MLU regime.

\begin{figure}
\begin{center}
\begin{tikzpicture}[x=12cm, y=1cm]
  \def\xA{0.0}     
  \def\xB{0.35}     
  \def\xC{0.70}     
  \def\xH{1.05}     
  \def\gap{0.005}   

  \draw[line width=0.4pt] (\xA,0) rectangle (\xH,1);

  \fill[green!24]   (\xA,0) rectangle (\xB,1);

  \fill[orange!22] (\xB+\gap,0) rectangle (\xC-\gap,1);

  \fill[red!14]    (\xC+\gap,0) rectangle (\xH,1);

  \draw[blue!60, line width=4pt] (\xB,0) -- (\xB,1);  
  \draw[brown!60, line width=4pt] (\xC,0) -- (\xC,1);  

  \filldraw[black] (\xB,1.25) circle (1.6pt);
  \node[anchor=south] at (\xB,1.25) {\scriptsize Critical MLU};
  \draw[-{Latex[length=2mm]}, gray!70] (\xB,1.2) -- (\xB,1.02);

  \filldraw[black] (\xC,1.25) circle (1.6pt);
  \node[anchor=south] at (\xC,1.25) {\scriptsize Lazy SDE};
  \draw[-{Latex[length=2mm]}, gray!70] (\xC,1.2) -- (\xC,1.02);

  \node at ({(\xA+\xB)/2},0.5) {\footnotesize {Maximal Local Updates} };
  \node at ({(\xB+\xC)/2},0.5) {\footnotesize  {Lazy ODE} };
  \node at ({(\xC+\xH)/2},0.5) {\footnotesize {Explosion} };

  \draw (\xA,0) -- ++(0,-0.10) node[below] {$0$};
  \draw (\xB,0) -- ++(0,-0.10) node[below] {$\Theta\big(\frac{\sqrt{D}}{LM}\big)$};
  \draw (\xC,0) -- ++(0,-0.10) node[below] {$\Theta\big(\frac{1}{\sqrt{LM}}\big)$};
  \draw (\xH,0) -- ++(0,-0.10) node[below] {$+\infty$};

  \node[below=7pt] at ({(\xA+\xH)/2},-0.55) {\small Residual Scale $\frac{\sigma_v}{ML}$};
\end{tikzpicture}
\end{center}

\caption{Phase diagram for the ResNets~\eqref{eq:parameterization-uv} as a function of the initialization scale $\sigma_u$ (for all shapes such that $D=O(LM)$. In the sub-critical MLU regime (green area), the behavior is asymptotically the same as if $\sigma_v=0$ (see Remark~\ref{rmk:subcritical}). Therefore, the critical MLU regime (blue) stands out as the only scaling with MLU and also maximal feature diversity.}\label{fig:sigmav-phase-diagram}
\end{figure}

\paragraph{A heuristic using operator norms.}  Let us mention another argument based on operator norms estimates that identifies the critical residual scale $\frac{\sigma_v}{LM}=\frac{\sqrt{D}}{LM}$. This argument is perhaps easier for intuition but only deals with upper-bounds so it is not sufficient by itself to justify the phase diagram (in contrast to the previous argument). It can be seen that the time-derivative of the output of the ResNet on a given input $x$ is given, at time $0$, by 
$$
\partial_t \hat h_0^L= \frac{1}{ML}\sum_{\ell=1}^L \sum_{j=1}^M \rho(\hat P_0^{j,\ell}) \partial_t \hat  V^{j,\ell}_0 + \frac{1}{ML}\sum_{\ell=1}^L \sum_{j=1}^M \underbrace{\rho'(\hat P_0^{j,\ell})\ipD{\partial_t \hat U^{j,\ell}_0}{\hat h^\ell_0}}_{\text{Local Feature Update $\delta_t \hat A_0^{j,\ell}$}} V_0^{j,\ell}+ \dots
$$
where we have ignored terms in $O(\sigma_v/\sqrt{ML})$ (the deviation of the first forward/backward passes from the identity map) and we recall that $\hat P_t^{j,\ell} \coloneqq \ipD{ \hat U^{j,\ell}_t}{\hat h^{\ell-1}_t}$.
Consider $\delta_t \hat P_t^{j,\ell}\coloneqq \ipD{\partial_t \hat U^{j,\ell}_t}{\hat h^\ell_t}$ the local pre-activation update and $\delta_t \hat A_t^{j,\ell}\coloneqq  \rho'(\hat P_t^{j,\ell})\ipD{\partial_t \hat U^{j,\ell}_t}{\hat h^\ell_t}$ the local feature update (both are of comparable magnitude). Let us study how large the local updates can be at time $0$ while maintaining the \emph{stability} property $\Vert \partial_t \hat h_0^L\Vert_{\RMS} =O(1)$.

Consider the matrix $V_0\in \RR^{D\times (ML)}$ and the vector $\delta_t \hat A_0\in \RR^{ML}$ built by aggregating the corresponding quantities for all units of all layers. The stability property requires
$$
\Big\Vert \frac{1}{ML} \sum_{\ell=1}^L \sum_{j=1}^M V_0^{j,\ell}\delta_t \hat A_0^{j,\ell}\Big\Vert_\RMS = \frac{1}{ML} \Vert V_0(\delta_t \hat A_0)\Vert_\RMS = O(1).
$$
By standard results on random matrices~\cite[Theorem 4.4.5]{vershynin2018high}, we have 
$$
\Vert V_0\Vert_{\RMS\to \RMS}=\frac{\sqrt{ML}}{\sqrt{D}}\Vert V_0\Vert_{2\to 2}=\Theta\Big( \sigma_v\Big(\frac{ML}{\sqrt{D}}+\sqrt{ML}\Big)\Big)=\Theta\Big(\sigma_v \frac{ML}{\sqrt{D}}\Big)
$$
where the last equality holds in the regime $D=O(ML)$. Hence, assuming that the operator norm bound $\Vert V_0(\delta_t \hat A_0)\Vert_\RMS\leq   \Vert V_0\Vert_{\RMS\to \RMS}\Vert \delta_t \hat A_0\Vert_\RMS$ is tight up to $\Theta(1)$ factors (which is in fact implied by our derivation above) we obtain that, if $D=O(ML)$, then
$$
\Vert \delta_t \hat A_0 \Vert_\RMS=O(\sqrt{D}/\sigma_v).
$$
This recovers informally the critical scaling $\sigma_v=\Theta(\sqrt{D})$ beyond which the local feature updates must asymptotically vanish to maintain stability.

\paragraph{Scaling of attention.} The phase  diagram for ResNets with attention blocks~\eqref{eq:attention} is the same than with 2LP blocks because, from the perspective of our analysis, they share a common structure. Indeed, for an input $X\in \RR^{D\times T}$ ($T$ tokens in $\RR^D$), the attention block is of the form $\phi_{\text{att}}(X,z)=W_O\psi(W_KX,W_QX,W_VX)$ where $\psi:\RR^{3d_k\times T}\to \RR^{d_k\times T}$ is a non-linear map and, if $d_k=1$, the parameters are vectors in $\RR^D$ (more generally, they are ``vector-like'' if $d_k$ is bounded while $D$ diverges). We therefore have the following scalings for the local feature update regime:
\begin{itemize}
\item $W_K,W_Q,W_V$ initialized with entrywise variance $\Theta_{d_k}(1/\sqrt{D})$ (taking into account that we did not insert a $1/D$ scaling factor in~\eqref{eq:attention}). These matrices play the role of $U$ in 2LP blocks.
\item $W_O$ initialized with entrywise variance $O_{d_k}(\sqrt{D})$ (if one also uses an explicit $1/(ML)$ branch scale, as in~\eqref{eq:discrete-forward}). This matrix plays the role of $V$ 2LP blocks.
\end{itemize}

\subsection{Convergence theorem with dimensional dependency}\label{sec:2LP-error}
We now derive an error bound between the ResNet and the Neural Mean ODE with 2LP blocks with an explicit dependency in $D$. For convenience, we limit ourselves to the MLU regime i.e., residual scales $O\big(\frac{\sqrt{D}}{LM}\big)$ (see Figure~\ref{fig:sigmav-phase-diagram}). 

The Mean ODE dynamics~\eqref{eq:GD-limit}, in the case of ResNets with 2LP blocks, can be written as follows. Let $\xi^u \sim (\mu_0^u)^{\otimes D}$ and $\xi^v \sim (\mu_0^v)^{\otimes D}$ and $\forall s\in [0,1]$,
\begin{align}\label{eq:limit-dynamics-2LP}
\left\{
\begin{aligned}
U_0(s)&=\xi^u \\
V_0(s)&=\xi^v 
\end{aligned}
\right. ,
&&
\left\{
\begin{aligned}
 U_{k+1}(s) &= U_k(s)-\eta_u \sum_{i=1}^n \rho'(P_{k,i}(s)){\color{red}\clip_u (}Q_{k,i}(s){\color{red})} h_{k,i}(s)\\
V_{k+1}(s) &= V_k(s) -\eta_v \sum_{i=1}^n \rho(P_{k,i}(s)) b_{k,i}(s) 
\end{aligned}
\right.
\end{align}
where 
\begin{align}
h_{k,i}(0) &= W_\tin x_i, & \partial_s h_{k,i}(s) &= \E[\rho(P_{k,i}(s)) V_k(s) ]\\
b_{k,i}(1) &= g_i(h_{k,i}(1)), & \partial_s b_{k,i}(s) &= - \E[\rho'(P_{k,i}(s)) {\color{red} \clip_b (} Q_{k,i}(s){\color{red})} U_k(s)]\\
P_{k,i}(s) &= \ipD{U_k(s)}{h_{k,i}(s)}, & Q_{k,i}(s) &= \ipD{V_k(s)}{ b_{k,i}(s)}
\end{align}
where $g_i$ is defined in Assumption~\ref{ass:high-dim} below.

\begin{remark} We have introduced clipping functions $\clip_u, \clip_b:\RR\to\RR$ which, in the vanilla GD dynamics, should be equal to the identity. We will always assume that $\clip_u$ is bounded (e.g. $\clip_u(x)= -1 \vee (x \wedge 1)$) since, otherwise, the error estimates suffer from an inevitable loss of integrability at each gradient step. We note that it is very common in practice to introduce some form of gradient clipping. We also obtain our tightest results when assuming $\clip_b$ is bounded, but since clipping in the backward pass is less usual -- and since the effect of not clipping at this level is less problematic -- we will also derive results covering $\clip_b=\id$ (at the cost of more stringent shape requirements).
\end{remark}

We make the following assumptions (where $B$ should be thought as independent of $D$).
\begin{assbox}
\begin{assumption}[High-dimensional assumptions]\label{ass:high-dim} There exists $B>0$ such that:
\begin{enumerate}
\item (Smooth activation) $\vert \rho(0)\vert \wedge \Vert \rho'\Vert_\infty \wedge \Vert \rho''\Vert_\infty\leq B$
\item (Normalized input) $\max_{i\in [1:n]} \Vert W_\tin x_i\Vert_\RMS \leq B$ 
\item (Smooth and normalized loss) The function $g_i:x\mapsto -W_\tout \nabla \loss_i(D^{-1} W_\tout^\top x)$ satisfies $\Vert g_i(0)\Vert_\RMS\leq B$ and is $B$-Lipschitz in RMS norm, for all $i\in [1:n]$.
\item (MLU regime) The initialization's standard deviation satisfies $\sigma_u \wedge \sigma_v \leq B\sqrt{D}$ and the LRs $\eta_u \wedge \eta_v \leq B$.
\item (Clipping functions) The functions $\clip_u,\clip_b:\RR\to\RR$ are $B$-Lipschitz and vanish at $0$. Moreover, $\clip_u$ has range in $[-B,B]$ (note that we do not exlude $\clip_b=\id$ for now).
\end{enumerate}
\end{assumption}
\end{assbox}

To state the convergence theorem, consider the ResNet's dynamics $(\hat Z^{j,\ell}_k = (\hat U^{j,\ell}_k,\hat V^{j,\ell}_k))_{k\geq 0}$ (defined in~\eqref{eq:2LP-update-finite}) and consider $M\times L$ independent copies of the limit dynamics $(Z^{j,\ell}_k =(U^{j,\ell}_k,V^{j,\ell}_k))_{k\geq 0}$ (defined in~\eqref{eq:limit-dynamics-2LP}) such that $U^{j,\ell}_0(s)=\hat U^{j,\ell}_0$ and $V^{j,\ell}_0(s)=\hat V^{j,\ell}_0$, $\forall s\in [0,1]$.
Consider the RMS distance between the two dynamics in parameter space, forward pass and backward pass respectively defined, with $s_\ell=\ell/L$ for $\ell\in [0:L]$, as
\begin{align*}
\Delta_k^U &\coloneqq \max_{\ell\in [1:L]} \Big(\frac{1}{M} \sum_{j=1}^M \Vert \hat U^{j,\ell}_k - U^{j,\ell}_k(s_{\ell-1})\Vert^2_\RMS\Big)^{1/2}, & \Delta_k^V &\coloneqq \max_{\ell\in [1:L]} \Big(\frac{1}{M} \sum_{j=1}^M \Vert \hat V^{j,\ell}_k - V^{j,\ell}_k(s_{\ell-1})\Vert^2_\RMS\Big)^{1/2}\\
\Delta_k^h &\coloneqq \max_{\substack{i\in [1:n] \\ \ell\in [0:L]}} \Vert \hat h_{k,i}^\ell - h_{k,i}(s_\ell)\Vert_\RMS,&
\Delta_k^b &\coloneqq \max_{\substack{i\in [1:n] \\ \ell\in [0:L]}} \Vert \hat b_{k,i}^\ell - b_{k,i}(s_\ell)\Vert_\RMS.
\end{align*}

\begin{thmbox}
\begin{theorem}[Error bound for ResNets with 2LP blocks with dimensional dependency]\label{claim:D-dependence}
Let Assumption~\ref{ass:high-dim} hold for $B>0$ and assume that the initial distribution $\mu_0^u$ (resp.~$\mu_0^v$) is $\bar \sigma_u^2$ (resp.~$\bar \sigma_v^2$) subgaussian, with $\bar \sigma^2_u \wedge \bar \sigma^2_v \leq B\sqrt{D}$. Suppose also that $\max\{D,\log(L)\}\leq BM$ and that the range of $\clip_b$ is included in $[-B,B]$.

Then for any $k\geq 0$, there exists $c_1,c_2>0$ that only depend on $B$ and $k$ such that for any $\delta > e^{-M}$, it holds with probability at least $1-\delta$ 
\begin{align}\label{eq:dimension-scaling}
\max_{k'\in [0:k]} \big\{ \Delta_{k'}^U,\Delta_{k'}^V,\Delta_{k'}^h, \Delta_{k'}^b\big\} \leq c_1 \left( \frac{1}{L}+ \frac{\sqrt{D}+\log(n/\delta)}{\sqrt{LM}}\right)
\end{align}
provided that the right-hand side is smaller than $c_2$.

If we do not assume that $\clip_b$ is bounded (e.g. $\clip_b=\id$) then the same conclusion holds under the additional conditions $\max\{D,\log(L)\}\leq B\sqrt{M}$ and $\delta > e^{-\sqrt{M}}$.
\end{theorem}
\end{thmbox}

We can make the following comments (we focus on the case $\clip_b$ bounded to fix ideas):
\begin{itemize}
\item  The theorem requires that the hidden width $M$ grows at least proportionally to the embedding dimension $D$ (or even $D^2=O(M)$ if $\clip_b=\id$). We believe that this condition could be replaced by  $D=O(LM)$ (or $D^2=O(LM)$ if $\clip_b=\id$) and we leave this potential improvement as an open question.
\item  When the right-hand side of~\eqref{eq:dimension-scaling} is large (or in events of probability smaller than $e^{-M}$),  the errors may begin to compound exponentially across depth and therefore increase dramatically. This phenomenon is visible as the red marks on Figure~\ref{fig:fluctuations}. In fact, we believe that in-expectation bounds are impossible in this setting (unless some stronger form of clipping is introduced) because of the existence of very rare events of catastrophic explosion.
\end{itemize}

\begin{remark}[Subcritical MLU regime]\label{rmk:subcritical}
For large $D$, the dynamics $(U_k,V_k)_k$ in the subcritical MLU regime   $\sigma_v=o(\sqrt{D})$ is the same as the dynamics $(\tilde U_k,\tilde V_k)_k$  obtained by taking $\sigma_v=0$ (all else being equal and with coupled randomness). To see this, let
\[
\tilde \Delta^Z_k \coloneq \sup_{s\in [0,1]} \E^{1/2}[\Vert U_k(s)-\tilde U_k(s)\Vert_\RMS^2+ \Vert V_k(s)-V_0(s)-\tilde V_k(s)\Vert_\RMS^2].
\]
Clearly, $\tilde \Delta^Z_0 = 0$, and by investigating the stability estimates of Section~\ref{eq:proof-limit-phase}, it is not difficult to obtain that  $\forall k\in \NN$, there exists $c_k$ independent of $\sigma_v$ and $D$ such that $\forall k'\leq k\vee \max\{ \tilde k\;;\; \tilde \Delta^Z_{\tilde k}\leq 1\}$, 
\begin{align*}
\tilde \Delta^Z_{k+1} \leq (1+\eta c_k)\tilde \Delta^Z_{k}+ c_k\sigma_v/\sqrt{D}.
\end{align*}
It follows, by  Gr\"onwall's lemma, that for any $k\geq 0$, there exists $c_1,c_2>0$ independent of $\sigma_v$ and $D$ such that 
$$
\sup_{k'\leq k} \tilde \Delta^Z_{k'}  \leq c_1\frac{\sigma_v}{\sqrt{D}}
$$
provided that the right-hand side is smaller than $c_2$. Of course, similar arguments show that the dynamics with $\sigma_u=o(\sqrt{D})$ is asymptotically the same as the one with $\sigma_u=0$.

As a consequence, in the subcritical MLU regime, two distinct units $(U^{j,\ell}_k,V^{j,\ell}_k)$ and $(U^{j',\ell}_k,V^{j',\ell}_k)$ at the same layer will represent the same feature throughout training as soon as $U^{j,\ell}_0= U^{j',\ell}_0$, because it is as if $V^{j,\ell}_0= V^{j',\ell}_0=0$. In contrast, at the critical MLU scaling, two such units will typically evolve differently. In that sense, there is more feature diversity in the critical scaling than in the subcritical scaling.
\end{remark}

\subsection{Experimental validation}
On Figure~\ref{fig:experiment-D} we compare our theoretical predictions with numerical experiments. The experimental setting is exactly as in Section~\ref{sec:experiments-LM}, but now we also make $D$ vary. We use the critical MLU regime scalings for initialization and LRs, as suggested by analysis in Section~\ref{sec:2LP-phase-diagram}. We observed that, as long as we are in the regime $D/(LM)=O(1)$ (the blue dots), there is a very good agreement between our theoretical predictions and the observed rates.

\begin{figure}
\centering
\begin{subfigure}{0.33\linewidth}
\centering
\includegraphics[scale=0.43]{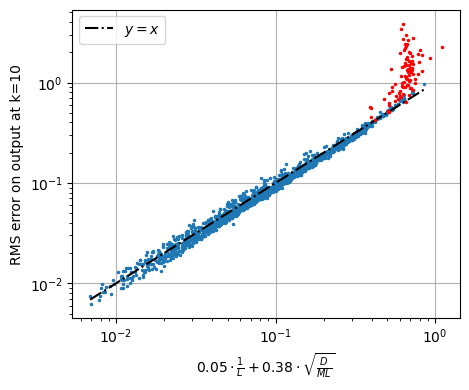}
\caption{$k=10$}\label{fig:fluctuations}
\end{subfigure}\hspace{0.05cm}%
\begin{subfigure}{0.33\linewidth}
\centering
\includegraphics[scale=0.43]{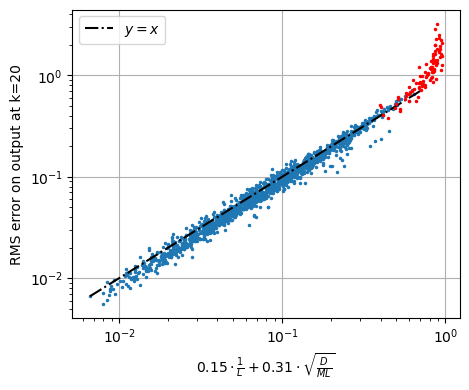}
\caption{$k=20$}\label{fig:laziness}
\end{subfigure}\hspace{0.05cm}%
\begin{subfigure}{0.33\linewidth}
\centering
\includegraphics[scale=0.43]{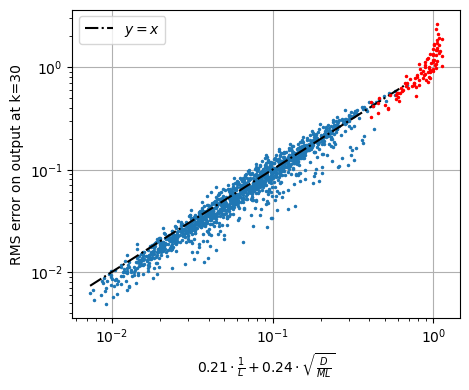}
\caption{$k=30$}\label{fig:all-losses}
\end{subfigure}
\caption{ RMS error on the output  between the ResNet and the limit (Mean ODE) model after $k$ GD steps. The red dots correspond to a large ratio $D/(ML)$ and are outside of the regime covered by our theory, while the blue dots are within the theory. We train 2100 ResNets of various shapes with $M\in \{1,\dots,200\}$, $L\in \{1,\dots, 200\}$ and $D\in \{1,\dots,100\}$ and compare the error with $a_k/L + b_k\sqrt{D/(LM)}$ where for each $k$, the coefficients $a_k$ and $b_k$ are estimated via least-squares fit. These results demonstrate the tightness of our analysis.}\label{fig:experiment-D}
\end{figure}

\section{Proofs for generic ResNets}\label{sec:proof-generic}

\subsection{Mean ODEs and regularity estimates}\label{sec:MODE-regularity}
Consider a generic Mean ODE of the form
\begin{align}\label{eq:generic-mean-ODE}
a(0)&\in \RR^D, & \dot a(s)&=F(s,a(s)), & F(s,x)\coloneqq \E[f(s,x,Z(s)]
\end{align}
where $f:[0,1]\times \RR^D\times \RR^p\to \RR^D$ and $(Z(s))_{s\in [0,1]}$ is a $\RR^p$-valued stochastic process. Let us start with a basic regularity lemma for generic ODEs which we will invoke repeatedly.
\begin{lemma}\label{lem:ODE-bound}
Let $\Vert \cdot\Vert$ be a norm on $\RR^D$. Assume that $F$ is continuous in $s$ and $L_x$-Lipschitz in $x$ (in the $\Vert \cdot \Vert$ norm) for some $L_x\geq 1$. Let $B= \max_{s\in [0,1]} \Vert F(s,0)\Vert $ and $R\coloneqq e^{L_x}(B+L_x\Vert a(0)\Vert)$. Then the mean ODE~\eqref{eq:generic-mean-ODE} has a unique solution $a:[0,1]\to \RR^d$ and it holds $\sup_{a\in [0,1]} \Vert a(s)\Vert\leq R$ and $s\mapsto a(s)$ is $R$-Lipschitz continuous. 
\end{lemma}
\begin{proof}
    Since $F$ is continuous in $s$ and $L_x$-Lipschitz in $x$, the ODE admits a unique global solution $a$ on $[0,1]$ by Picard-Lindel\"of theorem. Moreover, we have the linear growth control $\Vert \dot a(s)\Vert=\Vert F(s,a(s))\Vert\leq B+L_x\Vert a(s)\Vert$, and we obtain the following estimates by integrating: 
    \begin{align}
    \Vert a(s)\Vert& \leq e^{sL_x}\Vert a(0)\Vert+\frac{B}{L_x}(e^{sL_x}-1)\leq R, \\ \Vert \dot a(s)\Vert &\leq B +L_x\Vert a(s)\Vert \leq B+L_x(e^{sL_x}\Vert a(0)\Vert+\frac{B}{L_x}(e^{sL_x}-1))\leq R
    \end{align}
    which concludes the proof.
\end{proof}

We now come back to the training dynamics of the Neural Mean ODE~\eqref{eq:GD-limit} (with $\alpha=1$ here) and derive some regularity estimates. For this, it will be useful to remark that for each sample $x_i$, $i\in [1:n]$,
\begin{itemize}
\item the forward pass at iteration $k$ is a Mean ODE of the form~\eqref{eq:generic-mean-ODE} with $a(0)=x$ and $f=f^h_k:(s,x,z)\mapsto \phi(x,z)$;
\item the backward pass at iteration $k$ is, after depth-reversal (that is, composition with $s\mapsto 1-s$), a Mean ODE of the form~\eqref{eq:generic-mean-ODE} with $a(0)=w_{i,k}$ and $f=f^b_k:(s,x,z)\mapsto D_1\phi(h_{k,i}(s),z)^\top x$.
\end{itemize}

\begin{lemma}[Propagation of regularity]\label{lem:propagation-regularity} Let Assumption~\ref{ass:regularity} hold with $B>0$ and consider the Mean ODE dynamics~\eqref{eq:GD-limit} with $\mu_0\in \Pp(\RR^p)$ with first moment bounded by $B$ and $\alpha=1$. Then there exists $c$ that only depends on $B$ such that for all $k\geq 0$, $i\in [1:n]$,
\begin{enumerate}
\item[(i)] $Z_k$, $h_{k,i}$ and $b_{k,i}$  are uniquely well-defined;
\item[(ii)] the functions $s\mapsto h_{k,i}(s)$ and $s\mapsto b_{k,i}(s)$ are $c$-Lipschitz and bounded in norm by $c$;
\item[(iii)] the function $s\mapsto Z_k(s)$ is $(e^{ck\eta}-1)$-Lipschitz (surely).
\end{enumerate}
\end{lemma}

Although we do not explore the continuous-time limits $\eta\to0$ in this work, it should be noted that those regularity properties are preserved in the continuous-time limit, in particular because the estimate in (iii) only depends on $k\eta$. See~\citep{ding2022overparameterization, barboni2024understanding} for an analysis of the properties of the continuous time dynamics.

\begin{proof}
Let us prove the three claims by recursion over $k$. For $k=0$, $Z_0$ is constant by definition. Then, for $i\in [1:n]$, $s\mapsto h_{0,i}(s)$ is well-defined and Lipschitz by Lemma~\ref{lem:ODE-bound} applied to the function $f^h_0$ that is Lipschitz in $x$. In turn $s\mapsto b_{0,i}(s)$ is well-defined and Lipschitz by Lemma~\ref{lem:ODE-bound} applied to the function $f^b_0$ that is Lipschitz in $x$.

Now assume that the claims hold at $k\in \NN$ and that $s\mapsto Z_k(s)$ is $\Gamma_k$-Lipschitz. Since the map $\Update$ is Lipschitz when restricted to $b$ in a ball, there exists $c>0$ that depends only on $B$ such that
\begin{multline*}
\Vert Z_{k+1}(s)-Z_{k+1}(s')\Vert \leq \\
\Vert Z_{k}(s)-Z_{k}(s')\Vert +\eta \max_{i\in [1:n]} \Vert \Update(Z_k(s),h_k(s,x_i),b_k(s,x_i))- \Update(Z_k(s'),h_k(s',x_i),b_k(s',x_i)) \Vert\\
\leq (\Gamma_k + \eta c(\Gamma_k+1))|s-s'|.
\end{multline*}
Hence $s\mapsto Z_{k+1}(s)$ is $\Gamma_{k+1}$-Lipschitz with $\Gamma_{k+1}\leq \Gamma_k (1+c\eta)+\eta c$. Then we can apply Lemma~\ref{lem:ODE-bound} to obtain the well-posedness and Lipschitz regularity (independent of $k$) of $s\mapsto h_{k+1,i}(s)$ and of $s\mapsto b_{k+1,i}(s)$, in this order, $\forall i\in [1:n]$. This concludes the proof by recursion. For the expression of $\Gamma_k$: since $\Gamma_0=0$ and $\Gamma_{k+1}\leq \Gamma_k (1+c\eta)+c\eta$ by discrete Gr\"onwall's inequality\footnote{If $u_{k+1}\leq (1+\alpha)u_k+\beta$ for $k\geq 0$ $\alpha,\beta>0$, then $u_k\leq e^{k\alpha}(u_0+\beta/\alpha)-\beta/\alpha$.}, we get $\Gamma_k\leq e^{kc\eta}-1$.
\end{proof}

The next result controls the growth of the subgaussian variance-proxy (see definition in Section~\ref{sec:subgaussian}) of the parameters.
\begin{lemma}[Propagation of subgaussian tails]
\label{lem:propagation-subgaussianity} Let Assumption~\ref{ass:regularity} hold with $B>0$ and consider the Mean ODE dynamics~\eqref{eq:GD-limit} with $\mu_0\in \Pp(\RR^p)$ and $\alpha=1$. Assume moreover that $\mu_0$ is subgaussian. Then there exists $c>0$ that only depends on $B$ and $D$ such that $\forall k\geq 0$ and $s\in [0,1]$,
$$\Vert Z_k(s)\Vert_{vp}\leq e^{c k\eta}\Vert Z_0\Vert_{vp}. $$
\end{lemma}
\begin{proof}
By definition, $Z_0$ is subgaussian. As recalled in Section~\ref{sec:subgaussian}, if $X\in \RR^D$ is subgaussian and $f:\RR^D\to \RR^D$ is $L_f$-Lipschitz, then $\Vert f(X)\Vert_{vp}\leq c_0L_f\sqrt{D}\Vert X\Vert_{vp}$ for some absolute $c_0>0$. Therefore, using the fact that $z\mapsto \Update(z,h,b)$ is Lipschitz when $b$ is restricted to a ball (which it is by Lemma~\ref{lem:propagation-regularity}), there exists $c$ that only depends on $B$ and $D$ such that for $k\geq 0$ it holds (using the seminorm properties of $\Vert \cdot \Vert_{vp}$):
\begin{align*}
\Vert Z_{k+1}(s)\Vert_{vp} &\leq \Vert Z_{k}(s)\Vert_{vp}+\eta \max_{i\in [1:n]} \Vert \Update(Z_k(s),h_{k,i}(s),b_{k,i}(s))\Vert_{vp}
\leq (1+c\eta)\Vert Z_{k}(s)\Vert_{vp}.
\end{align*}
The result follows by recursion.
\end{proof}

\subsection{Stochastic approximation of Mean ODEs}\label{sec:stochastic-lemma}

In this section, we derive an approximation result which we will use to bound the error induced by each forward and backward pass of the training dynamics. It is inspired by results related to the so-called ``ODE method'' in the stochastic approximation literature~\citep{kushner2003stochastic}. This version of the result considers strong regularity assumptions and in order to cover the case of 2LP blocks, a more technical version of this lemma under weaker assumptions can be found in Lemma~\ref{lem:SA-relaxed}.

\begin{lemma}[Stochastic approximation of Mean ODEs]\label{lem:SA}
Let $f:[0,1]\times \RR^D\times \RR^p\to\RR^D$ be deterministic and let $Z: [0,1]\to \RR^D$ be a random function. Assume that $F$ is continuous in $s$ and that there exists $B$, $L_x$ and $\Gamma$ such that $x\mapsto \E[f(s,x,Z(s))]$ is $L_x$-Lipschitz and bounded by $B$ at $x=0$ (uniformly in $s\in [0,1]$) and $s\mapsto Z(s)$ is a.s.~$\Gamma$-Lipschitz. 
Let $R>0$ be given by Lemma~\ref{lem:ODE-bound} (depending only on $B$, $\Vert a(0)\Vert_2$ and $L_x$) such that the unique solution of the Mean ODE
\begin{align}
a(0)\in \RR^D,&& a'(s)=F(s,a(s)),&& F(s,x)\coloneqq \E[ f(s,x,Z(s))]
\end{align}
satisfies  $\sup_{s\in [0,1]}\Vert a(s)\Vert_2 \leq R$ and $s\mapsto a(s)$ is $R$-Lipschitz.
Moreover, we assume the following \emph{local} controls:
\begin{itemize}
\item there exist $\sigma_R>0$ such that for all $\Vert x\Vert_2\leq R$ and $s\in [0,1]$, the random variable $f(s,x,Z(s))$ is subgaussian with variance proxy $\sigma_R^2$.
\item there exist $L_R>0$ such that the function $f$ restricted to $\Vert x\Vert_2\leq R$ is $L_R$-Lipschitz (jointly in all its variables).
\end{itemize}
For integers $M,L\geq 1$, let $s_\ell\coloneqq \ell/L$ and consider the ``inexact Euler Monte-Carlo'' scheme
\begin{align}
\hat a^0\in \RR^D,&& \hat a^{\ell}=\hat a^{\ell-1}+\frac{1}{LM}\sum_{j=1}^M \hat f(s_{\ell-1},\hat a^{\ell-1},\hat Z^{j,\ell}),\quad \ell\in [1,L]
\end{align}
where $(\hat Z^{j,\ell})_{j,\ell}$ are random vectors and, for some error levels $\eps_0,\eps_1,\eps_2 \geq 0$, the discrete model satisfies:
\begin{itemize}
\item[(i)] Bounded initial mismatch: $\|\hat a^0-a(0)\|_2\le \eps_0$.
\item[(ii)] Bounded model error: $\sup_{z\in \RR^p}\sup_{\Vert x\Vert_2\leq 2R} \sup_{\ell \in [1:L]} \|\hat f(s_\ell,x,z)-f(s_\ell,x,z)\|_2\le \eps_1$ ;
\item[(iii)] Approximately independent parameters: there exists a family of independent samples $Z^{j,\ell}$ of $Z$ for ${j\in [1:M],\ell\in [1:L]}$ such that $\Vert\hat Z^{j,\ell}- Z^{j,\ell}(s_{\ell-1})\Vert_2\leq \eps_2$ a.s.
\end{itemize}
Then there exists $c_1,c_2>0$ that only depends on $L_x,L_R$, $\Gamma$ and $\Vert a(0)\Vert_2$ such that with probability at least $1-\delta$, it holds
\begin{align*}
\sup_{\ell\in [1:L]} \Vert \hat a^\ell -a(s_\ell)\Vert_2 \leq c_1\left(\eps_0+\eps_1+\eps_2+\frac{1}{L} + \sigma_R \frac{\sqrt{D}+\sqrt{\log(1/\delta)}}{\sqrt{LM}}\right)
\end{align*}
provided that the right-hand side is smaller than $c_2$.
\end{lemma}

\begin{proof} 
For $\ell\in [0:L-1]$, it holds
\begin{align*}
a(s_{\ell+1})-\hat a^{\ell+1} &= a(s_{\ell})-\hat a^{\ell}+\int_{s_{\ell}}^{s_{\ell+1}} \dot a(s)\d s -\frac{1}{ML} \sum_{j=1}^M \hat f(s_\ell,\hat a^{\ell},\hat Z^{j,\ell+1})\\
&= a(s_{\ell})-\hat a^{\ell} + \underbrace{\int_{s_{\ell}}^{s_{\ell+1}} \dot a(s)\d s - \frac1L F(s_{\ell}, a(s_{\ell}))}_{e_{euler}^{\ell+1}} \\
&\qquad +   \underbrace{ \Big(  \frac{1}{L}  F(s_\ell,a(s_\ell))-\frac{1}{ML}\sum_{j=1}^M f(s_\ell,a(s_\ell),Z^{j,\ell+1}(s_\ell))}_{e^{\ell+1}_{mc}} \Big)\\ &\qquad + \underbrace{\frac{1}{ML}\sum_{j=1}^M \big(f(s_\ell,a(s_\ell),Z^{j,\ell+1}(s_\ell))- \hat f(s_{\ell},\hat a^\ell,\hat Z^{j,\ell+1})\big)}_{e_{approx}^{\ell+1}}.
\end{align*}
By recursion, we have
\[
a(s_\ell)-\hat a^\ell = a(0)-\hat a^0 + \sum_{k=1}^{\ell} e_{euler}^{k} + \sum_{k=1}^{\ell} e_{mc}^{k} + \sum_{k=1}^{\ell} e_{approx}^{k}
\]
and therefore, with $\Delta^a_\ell\coloneqq \Vert a(s_\ell)-\hat a^\ell\Vert_2$, it holds
\begin{align}\label{eq:SA-proof-Delta}
\Delta^a_\ell \leq \Vert a(0)-\hat a^0 \Vert_2 +  \sum_{k=1}^{\ell} \Vert e_{euler}^{k}\Vert_2 +\sum_{k=1}^{\ell} \Vert e_{approx}^{k}\Vert_2 +  \Big\Vert   \sum_{k=1}^{\ell} e_{mc}^{k}\Big\Vert_2.
\end{align}
Note that for the Monte-Carlo error term, we take the norm \emph{after} summing across layers. Let us bound these error terms one by one. First, using the Lipschitz continuity of $f$, $a$ and $Z$, it holds for $\ell\in [0:L-1]$
\begin{align*}
\Vert e_{euler}^{\ell+1}\Vert_2 &=\Big\Vert \int_{s_\ell}^{s_{\ell+1}} \Big( F(s,a(s))-F(s_\ell,a(s_\ell))\Big)\d s\Big\Vert_2\\
&\leq \int_{s_\ell}^{s_{\ell+1}} \E\big[\Vert f(s,a(s),Z(s))- f(s_\ell,a(s_\ell),Z(s_\ell)))\Vert_2\big] \d s \\
&\leq L_R(1+ R+ \Gamma)\int_{s_\ell}^{s_{\ell+1}}  |s-s_\ell|\d s\leq \frac{c}{L^2}.
\end{align*}
for some $c>0$ that only depends on $R$, $L_R$ and $\Gamma$. Moreover, using the regularity of $f$ and Assumption (iii), \emph{under the assumption that $\Delta^a_{\ell-1}\leq R$} so that it holds $\Vert a(s_{\ell-1})\Vert_2, \Vert \hat a^{\ell-1}\Vert_2\leq 2R$, it holds 
\[
\Vert e_{approx}^\ell\Vert_2 \leq \frac{\eps_1 + L_R\Delta^a_{\ell-1} + L_R\eps_2}{L}.
\]
Finally, the random vectors $(e^\ell_{mc})_{\ell=1}^L$ are independent, centered and subgaussian with variance proxy $\sigma^2_R/(L^2M)$. It follows that $\sum_{k=1}^{\ell} e_{mc}^{k}$ is centered and subgaussian with variance proxy $\sigma^2_R \ell/(L^2M)$. By concentration of subgaussian vectors (Lemma~\ref{lem:subgaussian-vector-concentration}), there exists an absolute constant $c>0$ such that for all $\delta>0$ it holds with probability at least $1-\delta$
\[
\Big\Vert  \sum_{k=1}^{\ell} e_{mc}^{k}\Big\Vert_2 \leq c \sqrt{\frac{\sigma^2_R\ell}{L^2M}}(\sqrt{D}+\sqrt{\log(1/\delta)} \leq c \frac{\sigma_R(\sqrt{D}+\sqrt{\log(1/\delta)})}{\sqrt{ML}}
\]
By L\'evy-Ottaviani inequality (Lemma~\ref{lem:levy}), the same bounds holds for $\max_{\ell\leq L} \Big\Vert  \sum_{k=1}^{\ell} e_{mc}^{k}\Big\Vert_2 $ up to an absolute factor (this argument allows to avoid an extra $\sqrt{\log(L)}$ factor that a union bound over $\ell\in [1:L]$ would yield. Alternatively, one could directly apply Azuma-Hoeffding's lemma, see~\cite[Lemma~A.1]{mei2018mean}).

Plugging all these error estimates into~\eqref{eq:SA-proof-Delta}, we obtain that with probability at least $1-\delta$, for $\ell\in [1:L]$, provided $\Delta_k^a\leq R$ for $k\in [1:\ell-1]$, it holds
\[
\Delta_\ell^a \leq 
\frac{L_f}{L}\Big(\sum_{k=0}^{\ell-1} \Delta^a_k\Big) + \eps_0+\eps_1+L_f\eps_2+\frac{c}{L} + c \frac{\sigma(\sqrt{D}+\sqrt{\log(1/\delta))}}{\sqrt{LM}}
\]
where $c$ only depends on $L_R$, $L_f$, $\Gamma$ and $\Vert a(0)\Vert$. The result follows by recursion (the discrete Gr\"onwall lemma), and the condition on $\Delta_k^a$ is satisfied by requiring the upper-bound to be smaller than $R$, which is precisely what the control by $c_2$ in the statement achieves.
\end{proof}

\subsection{Proof of Theorem~\ref{thm:main}}\label{sec:proof-I}
\begin{itemize}
\item \textbf{Step 1: Set-up.}
In this proof, we denote by $c$ a positive real number that may depend on $B$, $k\eta$ and $D$, and that may change from line to line. First, by Lemma~\ref{lem:propagation-regularity}, there exists $R>0$ that only depends on $B$ such that 
\[
\max_{i\leq n,k'\leq k, s\in [0,1]} \big(\Vert h_{k',i}(s)\Vert_2 \wedge \Vert b_{k',i}(s)\Vert_2 \big)\leq R.
\]
Let $\tilde k\leq k$ be the largest index such that $\max_{i\leq n,k'\leq \tilde k,\ell\leq L} \big( \Vert \hat h_{k,i}^\ell \Vert_2 \wedge \Vert \hat b_{k,i}^\ell(s)\Vert_2\big)\leq 2R$ holds. We now work with time horizon $\tilde k$, and will show in the end of the proof that $\tilde k$ can be made equal to $k$ with high-probability provided $c_2$ is small enough. Recall that
\begin{align*}
\hat Z^{j,\ell}_{k+1}&=\hat Z^{j,\ell}_k + \frac{\eta}{n} \sum_{i=1}^n \Update(\hat Z^{j,\ell}_k, \hat h^{\ell-1}_{k,i}, \hat b^\ell_{k,i}),
\\
Z_{k+1}(s) &= Z_{k}(s) + \frac{\eta}{n}\sum_{i=1}^n \Update(Z_k(s), h_{k,i}(s), b_{k,i}(s))
\end{align*}
for the same map $\Update(z,h,b)=-D_2\phi(h,z)^\top b$. Under Assumption~\ref{ass:regularity}, the map $\Update$ restricted to $b$ in a ball of radius $2R$ is $L_R$-Lipschitz. Therefore, $\forall k\leq \tilde k$,
\begin{align}\label{eq:proof-iteration-inequality}
\Delta^Z_{k+1}\leq \Delta^Z_{k}+\eta L_R(\Delta^Z_{k}+\Delta^h_{k}+\Delta^b_{k}).
\end{align}
Let us now fix $k\leq \tilde k$ and control these various terms with high-probability.

\item \textbf{Step 2: Control on $\Delta_k^h$.} Let us verify that we can apply Proposition~\ref{lem:SA} with $f=f_k^h:(s,x,z)\mapsto \phi(a,z)$ and $Z=Z_k$. Clearly $f_k^h$ is Lipschitz and by Lemma~\ref{lem:propagation-regularity}, $s\mapsto Z_k(s)$ also. By Lemma~\ref{lem:propagation-subgaussianity} and composing with a Lipschitz function, $f_k^h(s,x,Z_k(s))$ is subgaussian with variance proxy $c\sigma_0^2$ (where $c$ depends only on $B$, $D$ and $k\eta$). Therefore the proposition applies (with $\epsilon_0=\epsilon_1=0$ and $\epsilon_2=\Delta^Z_k$). By a union bound over $i\in [1:n]$, there exists $c_{1,k}$ such that with probability at least $1-\delta$, it holds
\[
\Delta_k^h \leq c_{1,k}\Big(\Delta_k^Z+\frac1L +\frac{\sigma_0(1+\sqrt{\log(n/\delta)})}{\sqrt{ML}}\Big).
\]
\item \textbf{Step 3: Control on $\Delta_k^b$.} Let us verify that we can apply Proposition~\ref{lem:SA} with $f_k^b:(s,x,z)\mapsto D_1\phi(h_k(s,x_i),z)^\top x$, $Z=Z_k$ and $\hat f_k^b:(s_\ell,x,z)\mapsto D_1\phi(\hat h_k^\ell(x_i),z)^\top x$. Although $f_k^b$ is not globally Lipschitz, it is Lipschitz when its argument $x$ is restricted to a ball by Lemma~\ref{lem:propagation-subgaussianity} (the assumptions of Proposition~\ref{lem:SA} are designed precisely to cover this case). Also, by the subgaussian tail estimates of Lemma~\ref{lem:propagation-subgaussianity}, $f_k^b(s,x,Z_k(s))$ is subgaussian with variance proxy $c\sigma_0^2$ (where $c$ depends only on $B$, $D$ and $k\eta$). Therefore the proposition applies (with $\epsilon_0=\Vert w_{i,k}-\hat w_{i,k}\Vert_2\leq c\Delta_k^h$, $\epsilon_1\leq c \Delta_k^h$ and $\epsilon_2\leq \Delta_k^Z$). By a union bound over $i\in [1:n]$, there exists $c_{2,k}$ such that with probability at least $1-\delta$, it holds
\[
\Delta_k^b \leq c_{2,k}\Big(\Delta_k^Z+\Delta_k^h+\frac1L +\frac{\sigma_0(1+\sqrt{\log(n/\delta)})}{\sqrt{ML}}\Big).
\]
\item \textbf{Step 4: Conclusion.} Take a union bound over the at most $2\times k$ events where all the previous bounds hold for $k'\leq \tilde k \vee (k-1)$. Plugging into~\eqref{eq:proof-iteration-inequality}, there exists $c$ such that with probability at least $1-\delta$, for $0\leq k'\leq \tilde k \vee (k-1)$, it holds
\[
\Delta_{k'+1}^Z \leq  \Delta_{k'}^Z + \eta c \Big(\Delta_{k'}^Z+\frac1L +\frac{\sigma_0(1+\sqrt{\log(kn/\delta)})}{\sqrt{ML}}\Big).
\]
Since $\Delta^Z_0=0$, the conclusion for $(\Delta^Z_k)$ follows by Gr\"onwall's inequality, and for $(\Delta^h_k)$ and $(\Delta^b_k)$ by the bounds in Step 2 and Step 3 respectively. Finally, by taking $c_2$ small enough in the statement of the theorem, we can ensure $\tilde k\geq k$ under the same event. This concludes the proof.
\end{itemize}

\subsection{Proof of Theorem~\ref{thm:main-lazy}}\label{sec:proof-II}

The general structure of the proof follows that of the proof of Theorem~\ref{thm:main}. However there is an additional error term arising from the linearization of $\phi$ in its second argument, and the scaling of the noise variance differs. In this proof, we denote by $c$ a positive real number that depends only on $B$, $k\eta$ and $D$, and that may change from line to line.
\begin{itemize}
\item \textbf{Step 1: set-up.}
The functions appearing in the Mean ODE of the limit model are:
\begin{itemize}
\item Forward pass: $f_k^{\underline h}:(s,x,z)=D_2\phi(x,z_0)z_1$ where $z=(z_0,z_1)\in \RR^{2p}$ 
\item Backward pass: $f_k^{\underline b}:(s,x,z)=D_{1,2}\phi(\underline h_{k,i}(s),z_0)^{*1}[x,z_1]$ where $z=(z_0,z_1)\in \RR^{2p}$ 
\end{itemize}

Under Assumption~\ref{ass:lazy-regularity}, one can prove using similar arguments as for the $\alpha=1$ limit that $\underline h_{k,i}$ and $\underline b_{k,i}$ have the same regularity properties as $h_{k,i}$ and $b_{k,i}$, respectively.
The main difference is that one should complement the recursion of Lemma~\ref{lem:propagation-regularity} with the property that $\sup_{s\in [0,1]} \Vert \zeta_k(s)\Vert $ is bounded by a constant depending on $B$ and $k\eta$ only (this follows from the boundedness of the update map along the iterations of GD). We can also bound the subgaussian norm of  $Z_k(s)=Z_0+\alpha^{-1}\zeta_k(s)$  by $c\sigma_0$ using the approach of Lemma~\ref{lem:propagation-subgaussianity}.

As in the proof of Theorem~\ref{thm:main}, we consider a (random) index $\tilde k$ such that the norms of $\underline h_{k',i}, \underline b_{k',i},\hat h_{k',i}$ and $\hat b_{k',i}$ are controlled by $2R$ for all $k'\leq \tilde k \vee k$ and $i\in [1:n]$. Recalling
\[
\Delta_k^\zeta \coloneqq \max_{\substack{j\in [1:M] \\ \ell\in [1:L]}} \Vert \alpha(\hat Z_k^{j,\ell}-\hat Z_0^{j,\ell}) -\zeta_k^{j,\ell}(s_{\ell-1}) \Vert
\]
and the expression of the updates, it holds (a factor $\alpha$ gets canceled out):
\begin{align}\label{eq:proof-lazy-step}
\Delta_{k+1}^\zeta \leq \Delta_{k}^\zeta + c\eta(\Delta_{k}^{\underline h}+\Delta_{k}^{\underline b}).
\end{align}
\item \textbf{Step 2: Control on $\Delta_k^{\underline h}$.} Here we would like to apply Proposition~\ref{lem:SA}, but it does not apply as such, due to the fact that the regularity estimates of the map $\alpha\phi$ diverges. Instead of the error decomposition in the beginning of the proof of Proposition~\ref{lem:SA}, we consider the following decomposition:
\begin{align*}
\underline h_{k,i}(s_{\ell+1})- \hat h^{\ell+1}_{k,i} & = \underline h_{k,i}(s_{\ell})- \hat h^{\ell}_{k,i} \\
&+  \int_{s_\ell}^{s_{\ell+1}} \E[ D_2\phi (\underline  h_{k,i}(s),Z_0) \zeta_k(s)]\d s - \frac1L \E[ D_2\phi (\underline  h_{k,i}(s_\ell),Z_0) \zeta_{k}(s_\ell)] \\
&+ \frac1L \E[ D_2\phi (\underline h_{k,i}(s_\ell),Z_0) \zeta_{k}(s)] -  \frac1L \E[ D_2\phi ({\color{red}\hat h}_{k,i}^\ell,Z_0) \zeta_k(s)] \\
&+ \frac1L \E[ D_2\phi (\hat h_{k,i}^\ell,Z_0) \zeta_{k}(s)] - \frac{\alpha}{L} \E[\phi(\hat h_{k,i}^\ell,Z_0+\alpha^{-1}\zeta_k(s_\ell))]\\
&+\frac{\alpha}{L} \E[\phi(\hat h_{k,i}^\ell,Z_0+\alpha^{-1}\zeta_k(s_\ell))]-\frac{\alpha}{ML} \sum_{j=1}^M  \phi(\hat h_k^\ell,Z^{j,\ell+1}_0+\alpha^{-1}\zeta_k^{j,\ell+1})\\
&+ \frac{\alpha}{ML} \sum_{j=1}^M  \Big(\phi(\hat h_k^\ell,Z^{j,\ell+1}_0+\alpha^{-1}\zeta_k^{j,\ell+1})-\phi(\hat h_{k,i}^\ell,Z^{j,\ell+1}_0+\alpha^{-1}{\color{red}\hat \zeta}_k^{j,\ell+1})\Big)
\end{align*}
 The terms from line to line are respectively (i) the Euler discretization error bounded by $c/L^2$, (ii) a term bounded by $(c/L)\Vert \underline h_{k,i}(s_{\ell})- \hat h^{\ell}_{k,i} \Vert_2$, (iii) the linearization error bounded by $c/(\alpha L)$, (iv) the Monte-Carlo sampling error multiplied by $\alpha$ and (v) a term bounded by $c\Delta_k^{\zeta}/L$.  Of all these terms, only the linearization term (iii) is new compared to the proof of Proposition~\ref{lem:SA}. There is also a difference in the handling of the Monte-Carlo term (iv) because here $\hat h^\ell_{k,i}$ is not independent from $(Z_0^{j,\ell+1},\zeta_k^{j,\ell+1})$. This can be handled by using a \emph{uniform} concentration bound over $\hat h^\ell_k$ in the ball of radius $2R$ in $\RR^D$ which leads to the same high probability bound up to terms depending on $D$ only.
 
 Then, proceeding as in the proof of Proposition~\ref{lem:SA}, we get that with probability at least $1-\delta$, 
\[
\Delta_k^{\underline h}\leq c_{1,k}\Big( \Delta_k^{\zeta} +\frac{1}{\alpha} +\frac1L +\frac{\alpha \sigma_0 (1+\sqrt{\log(n/\delta)})}{\sqrt{LM}}\Big).
\]

Next, to bound $\Delta_k^{\underline b}$, we can proceed along the same line as in Step 3 of the proof of Theorem~\ref{thm:main} and using an error decomposition similar to the one above. We obtain that with probability at least $1-\delta$, 
\[
\Delta_k^{\underline b}\leq c_{2,k}\Big( \Delta_k^{\zeta}+\Delta_k^{\underline h} +\frac{1}{\alpha} +\frac1L +\frac{\alpha \sigma_0 (1+\sqrt{\log(n/\delta)})}{\sqrt{LM}}\Big).
\]
\item \textbf{Conclusion.} We finally take a union bound over the at most $2k$ events where the previous bounds hold up to $\tilde k \vee (k-1)$, we plug these estimates into~\eqref{eq:proof-lazy-step} and conclude exactly as in the proof of Theorem~\ref{thm:main} (ensuring in particular that $\tilde k\geq k$ by taking $c_2$ small enough).
\end{itemize}

\section{Proofs for ResNets with 2LP blocks}\label{sec:proof-2LP}
In this whole section, we work under  Assumption~\ref{ass:high-dim} with $B>0$, and we denote by $c$ a positive real number that may depend on $B$ and the number of iterations $k$, and that may change from line to line.

\subsection{High-dimensional regularity estimates on the Mean ODE}\label{eq:proof-limit-phase}
We first state a simple probability result that is key to track the scale of various quantities in the limit model.

\begin{lemma}[Decorrelation lemma]\label{lem:scalar-vector-expectation}
Let $(X,Y)\in \RR^D\times \RR$ be a pair of $L^2$ random variables. 
\begin{enumerate}[label=(\roman*)]
\item 
Then
\begin{align}
\Vert \E[XY]\Vert_{2} \leq \E^{1/2}[\Vert X\Vert_2^2 ] \cdot \E^{1/2} [Y^2].
\end{align}
\item If moreover the coordinates of $X$ satisfy $\E[X[i]X[j]]= 0$ if $i\neq j$ and $\E[X[i]^2]\leq \sigma^2$, then 
\begin{align}
\Vert \E[XY]\Vert_{2} \leq \sigma \E^{1/2} [Y^2].
\end{align}
In particular, if the entries of $X$ are  zero mean, of variance bounded by $\sigma^2$ and independent (but not necessarily independent of $Y$) it holds $\Vert \E[XY]\Vert_{\RMS} \leq \frac{\sigma}{\sqrt{D}} \E^{1/2} [Y^2]$.
\end{enumerate}
\end{lemma}
The last claim plays an important role in  the stability of the Mean ODE in the residual scaling $\Theta\big(\frac{\sqrt{D}}{LM}\big)$. In words, it states that a scalar random variable cannot fully correlate with all the entries of a random vector with independent entries.
\begin{proof}
 The first property is direct by Cauchy-Schwartz inequality applied entrywise
$$
\Vert \E[XY]\Vert_{2}^2 = \sum_{i=1}^D |\E[X[i]Y]|^2\leq \sum_{i=1}^D \E[ X[i]^2]\cdot  \E [Y^2]= \E[\Vert X\Vert_2^2] \cdot \E [Y^2].
$$
For the second property,
\begin{align*}
\Vert \E[XY]\Vert_2 = \sup_{\Vert z\Vert_{2}\leq 1} \E[Y z^\top X]\leq \sup_{\Vert z\Vert_{2}\leq 1} \E^{1/2}[Y^2] \cdot \E^{1/2}[(z^\top X)^2]
\end{align*}
but since the entries of $X$ are uncorrelated, we have for any fixed $z\in \RR^D$
\begin{equation*}
\E[(z^\top X)^2] =\E\Big[ \big(\sum_{i=1}^D z[i]X[i]\big)^2\Big] = \sum_{i=1}^D \E[z[i]^2 X[i]^2]\leq \sigma^2 \Vert z\Vert_2^2.\qedhere
\end{equation*}
\end{proof}

We now prove uniform-in-$D$ stability estimates for the Mean ODE dynamics, where the decorrelation lemma will play a key role. Let us introduce the convenient notations
\begin{align*}
\Delta U_k(s) &= U_k(s)- U_0(s), & \Delta V_k(s) &= V_k(s)- V_0(s).
\end{align*}

The goal of the next proposition is to obtain regularity estimates on the forward and backward pass. The other estimates will be improved in the next statements under stronger tail assumptions on the initialization.

\begin{proposition}[Large-dimensional stability of the Mean ODE dynamics]\label{thm:phase-diagram-D} 
Let Assumption~\ref{ass:high-dim} hold for some $B>0$ (note: we do not use the boundedness of $\clip_u$ or $\clip_b$ here) and consider the Mean ODE dynamics~\eqref{eq:limit-dynamics-2LP}. Then for all $k\geq 0$, there exists $c>0$ that only depends on $B$ and $k$ such that for all $i\in [1:n]$, $s\in [0,1]$,  
\begin{align*}
\Vert h_{k,i}(s)\Vert_\RMS\leq c, && \Vert b_{k,i}(s)\Vert_\RMS\leq c\\
\E^{1/2}[\Vert \Delta U_k(s)\Vert_\RMS^2] \leq c && \E^{1/2}[\Vert \Delta V_k(s)\Vert_\RMS^2] \leq c\\
\E^{1/2}[\vert P_{k,i}(s)\vert^2] \leq c && \E^{1/2}[\vert Q_{k,i}(s)\vert^2] \leq c
\end{align*}
Moreover, the functions $s\mapsto h_{k,i}(s)$ and  $s\mapsto b_{k,i}(s)$ are $c$-Lipschitz in RMS norm on $[0,1]$. 
\end{proposition}
\begin{proof}
Let us prove the result by recursion over $k$.

\textbf{Initial step.} The first forward and backward passes are trivial and given by 
\begin{align}
h_{0,i}(s) = W_\tin x_i,&& b_{0,i}(s) = g_i(W_\tin x_i), && \forall  s\in[0,1].
\end{align} 
Moreover, using that the coordinates of $U_0$ are zero mean and independent,
\begin{align}\label{eq:proof-stability-D-1}
\E[ |P_{0,i}(s)|^2 ] =\frac{1}{D^2} \E \Big[\sum_{d=1}^D U_0[d]^2(W_\tin x_i) [d]^2\Big] = \frac{\sigma_u^2}{D^2}\Vert W_\tin x_i\Vert_2^2 \leq c
\end{align}
and a similar bound holds for $\E[ |Q_{0,i}(s)|^2 ]$. Also, $\Delta U_0(s)=\Delta V_0(s)=0$. Therefore the proposition is true for $k=0$.

\textbf{Inductive step.}  Let us assume that the proposition is true at iteration $k-1$ for some $c$, and let us use $c'$ for scalars that depend only on $c$ and $B$ (but, crucially, independent of $D$). 
Using that $\rho'$ is bounded, we have, 
\begin{align*}
\E^{1/2}[\Vert U_{k}(s)-U_{k-1}(s)\Vert_\RMS^2] &= \frac{\eta}{n}\sum_{i=1}^n \E^{1/2}[\Vert \rho'(P_{k-1,i}(s))\clip_u(Q_{k-1,i}(s))h_{k-1,i}(s)\Vert^2_\RMS] \leq c',\\
\E^{1/2}[\Vert V_{k}(s)-V_{k-1}(s)\Vert_\RMS^2] &= \frac{\eta}{n}\sum_{i=1}^n \E^{1/2}[\Vert \rho(P_{k-1,i}(s))b_{k-1,i}(s)\Vert^2_\RMS] \leq c',
\end{align*}
from which it follows $\E^{1/2}[\Vert \Delta U_{k}(s)\Vert_\RMS^2] \wedge \E^{1/2}[\Vert \Delta V_{k}(s)\Vert_\RMS^2] \leq c'$. Therefore it holds
 \begin{align*}
 \E^{1/2} [P_{k,i}(s)^2] &\leq  D^{-1}\E^{1/2}[U_0^\top h_{k,i}(s)] + D^{-1}\E^{1/2}[\Delta U_k(s)^\top h_{k,i}(s)]\\
 &\leq c' \Vert h_{k,i}(s)\Vert_{\RMS}.
 \end{align*}
 where we used the computations~\eqref{eq:proof-stability-D-1} for the term involving $U_0$, and analogously $ \E^{1/2} [Q_{k,i}(s)^2] \leq c' \Vert b_{k,i}(s)\Vert_{\RMS}$.
Now to control the forward pass, we have
\begin{align*}
\Vert \partial_s h_{k,i}(s)\Vert_\RMS &= \Vert \E[\rho(P_k(s)) V_k(s)]\Vert_\RMS\\
&\leq \E [\Vert  \rho(P_k(s)) \Delta V_k(s)\Vert_\RMS] + \Vert \E[\rho(P_k(s))V_0(s)]\Vert_\RMS \\
&\leq  \E^{1/2}[ \Vert \Delta V_{k}\Vert^2_\RMS]\cdot \E^{1/2}[ \rho(P_k(s))^2] +\frac{\sigma_v}{\sqrt{D}} \E^{1/2}[ \rho(P_k(s))^2]
\end{align*}
where we have used Lemma~\ref{lem:scalar-vector-expectation}-(i) and (ii) to bound the first and second term, respectively. We deduce the bound
\begin{align}
\Vert \partial_s h_{k,i}(s)\Vert_\RMS \leq c'(1+\Vert h_{k,i}(s)\Vert_\RMS).
\end{align}
By Gr\"onwall's lemma, and since the initial point satisfies $\Vert W_\tin x_i\Vert_\RMS\leq c$, we deduce that there exists $c'>0$ such that  $\Vert h_{k,i}(s)\Vert_\RMS\leq c'$ and $s\mapsto h_{k,i}(s)$ is $c'$-Lipschitz in RMS norm.

Finally, to control the backward pass, we have 
\begin{align*}
\Vert \partial_s b_{k,i}(s)\Vert_\RMS &= \Vert \E[\rho'(P_k(s))\clip_b(Q_k(s))U_k(s)]\Vert_\RMS\\
&\leq \E [\Vert \rho'(P_k(s))\clip_b(Q_k(s)) \Delta U_k(s)\Vert_\RMS]+ \Vert \E[\rho'(P_k(s))\clip_b(Q_k(s)) U_0(s)]\Vert_\RMS  \\
&\leq  \E^{1/2}[ \Vert \Delta V_{k}\Vert^2_\RMS]\cdot \E^{1/2}[ \rho'(P_k(s))^2Q_k(s)^2] +\frac{\sigma_v}{\sqrt{D}} \E^{1/2}[ \rho'(P_k(s))^2Q_k(s)^2]
\end{align*}
using again Lemma~\ref{lem:scalar-vector-expectation}-(i) and (ii) to bound the first and second term, respectively.
We deduce the bound
\begin{align}
\Vert \partial_s b_{k,i}(s)\Vert_\RMS \leq c'\Vert b_{k,i}(s)\Vert_\RMS.
\end{align}
By Gr\"onwall's lemma, and since the initial point satisfies $\Vert g_i( h_{k,i}(1))\Vert_\RMS\leq c'$ under Assumption~\ref{ass:high-dim}, we deduce that there exists $c'>0$ such that  $\Vert b_{k,i}(s)\Vert_\RMS\leq c'$ and $s\mapsto b_{k,i}(s)$ is $c'$-Lipschitz in RMS norm. The result follows by recursion.
\end{proof}

\begin{lemma}[Propagation of subgaussian tails]\label{lem:subgaussian-bis}
Let Assumption~\ref{ass:high-dim} hold for some $B>0$ and consider the Mean ODE dynamics~\eqref{eq:limit-dynamics-2LP} (note: we do not use that $\clip_u$ or $\clip_b$ are bounded here). Suppose moreover that there exists $\bar \sigma_0 \leq B\sqrt{D}$ such that $\Vert U_0\Vert_{\psi_2}, \Vert V_0\Vert_{\psi_2}\leq \bar \sigma_0$. Then for any $k\geq 0$, there exists $\kappa_k\geq 0$ that only depends on $B$ and $k$ such that 
\begin{align*}
\forall s\in [0,1],\quad  \Vert \Vert \Delta U_k(s)\Vert_\RMS \Vert_{\psi_2}\wedge \Vert \Vert \Delta V_k(s)\Vert_\RMS \Vert_{\psi_2} \leq \kappa_k, && \Vert P_{i,k}(s)\Vert_{\psi_2}\wedge \Vert Q_{i,k}(s)\Vert_{\psi_2} \leq \kappa_k.
\end{align*}
In particular, it follows
$\Vert   \Delta U_k(s)\Vert_{\psi_2}\wedge
\Vert   \Delta V_k(s)\Vert_{\psi_2}\wedge
\Vert   U_k(s)\Vert_{\psi_2}\wedge 
\Vert  V_k(s)\Vert_{\psi_2}\leq \kappa_k\sqrt{D}$.
\end{lemma}
\begin{proof}
The second claim follows from the first one because of the inequality $\Vert X\Vert_{\psi_2}\leq \sqrt{D}\Vert \Vert X\Vert_\RMS\Vert_{\psi_2}$ that holds for any subgaussian random vector $X\in \RR^D$ (see details in Appendix~\ref{sec:subgaussian}), and also using $\Vert U_k(s)\Vert_{\psi_2}\leq \Vert U_0(s)\Vert_{\psi_2}+\Vert \Delta U_k(s)\Vert_{\psi_2}$. 

Let us prove the first claim by recursion. The case $k=0$ holds because $\Delta U_k(s)=\Delta V_k(s)=0$ and  also
\[
\Vert P_{0,i}(s)\Vert_{\psi_2} \leq c \frac{\bar \sigma_0}{\sqrt{D}} \Vert h_{0,i}\Vert_\RMS\leq c'
\]
and $\Vert Q_{0,i}(s)\Vert_{\psi_2}$ can be bounded similarly, for some $c,c'$ that only depend on $B$.

Assume that the property holds at $k\geq 0$ for some $\kappa_{k}$ and let us use $c$ for a scalar that only depends on $B$ and $\kappa_{k}$. Using the recursion hypothesis, the boundedness of $\rho'$ and the bounds from Proposition~\ref{thm:phase-diagram-D}, it holds
\begin{align*}
\Vert \Vert \Delta U_{k+1}(s)\Vert_{\RMS}\Vert_{\psi_2} &\leq \Vert \Vert \Delta U_{k}(s)\Vert_{\RMS}\Vert_{\psi_2}+ \frac{\eta_u}{n}\sum_{i=1}^n \Vert \Vert \rho'(P_{k,i}(s))\clip_u(Q_{k,i}(s))h_{k,i}(s)\Vert_{\RMS}\Vert_{\psi_2}\\
&\leq \Vert \Vert \Delta U_{k}(s)\Vert_{\RMS}\Vert_{\psi_2} + \frac{c}{n}\sum_{i=1}^n \Vert h_{k,i}(s)\Vert_{\RMS}\cdot \Vert  Q_{k,i}(s))\Vert_{\psi_2} \leq  \kappa_k+ c\kappa_k.
\end{align*}

Similarly, 
\begin{align*}
\Vert \Vert \Delta V_{k+1}(s)\Vert_{\RMS}\Vert_{\psi_2} &\leq \Vert \Vert \Delta V_{k}(s)\Vert_{\RMS}\Vert_{\psi_2}+ \frac{\eta_v}{n}\sum_{i=1}^n \Vert \Vert \rho(P_{k,i}(s))b_{k,i}(s)\Vert_{\RMS}\Vert_{\psi_2}\\
&\leq \Vert \Vert \Delta V_{k})(s)\Vert_{\RMS}\Vert_{\psi_2}+ \frac{c}{n}\sum_{i=1}^n (1+\Vert  P_{k,i}(s)\Vert_{\psi_2}) \Vert b_k(s)\Vert_{\RMS} \leq  \kappa_k+c(1+\kappa_k)
\end{align*}
where we used that by the Lipschitz continuity of the univariate map $\rho$, $\Vert \rho( P_{k,i}(s))\Vert_{\psi_2}\leq c(1+\Vert  P_{k,i}(s)\Vert_{\psi_2})$. From these bounds, we directly deduce 
\[
\Vert   \Delta U_{k+1}(s)\Vert_{\psi_2}\wedge
\Vert   \Delta V_{k+1}(s)\Vert_{\psi_2}\wedge
\Vert   U_{k+1}(s)\Vert_{\psi_2}\wedge 
\Vert  V_{k+1}(s)\Vert_{\psi_2}\leq c(\kappa_k+1)\sqrt{D}.
\]

Finally, using the definition of vector subgaussian norm, we have
 \begin{align*}
 \Vert P_{k+1,i}(s)\Vert_{\psi_2} \leq    D^{-1}\Vert U_{k+1}(s)^\top h_{k+1,i}(s)\Vert_{\psi_2}\leq \frac{c}{\sqrt{D}}\Vert U_{k+1}(s) \Vert_{\psi_2}\leq c'(\kappa_k+1)
 \end{align*}
 and analogously $\Vert Q_{k+1,i}(s)\Vert_{\psi_2}\leq c'(\kappa_k+1)$. This proves the existence of $\kappa_{k+1}$ with the suitable dependency, and concludes the proof.
\end{proof}

We now prove propagation of Lispschitz regularity. We do so in subgaussian norm as this will be convenient later, but an analogous statement holds in any $L^p$ norm under $L^p$ initialization (here we crucially use the boundedness of $\clip_u$).
\begin{lemma}[Propagation of Lipschitz regularity]\label{lem:Lipschitz-bis}
Let Assumption~\ref{ass:high-dim} hold for some $B>0$. Suppose moreover that there exists $\bar \sigma_0 \leq B\sqrt{D}$ such that $\Vert U_0\Vert_{\psi_2}, \Vert V_0\Vert_{\psi_2}\leq \bar \sigma_0$. Consider the Mean ODE dynamics~\eqref{eq:limit-dynamics-2LP}. Then for all $k\geq 0$, there exists $\Gamma_k\geq 0$ that only depends on $B$ and $k$ such that for all $i\in [1:n]$, $s\in [0,1]$, 
\begin{align*}
\Vert P_{k,i}(s)-P_{k,i}(s')\Vert_{\psi_2} &\leq \Gamma_k |s-s'|,&
\Vert Q_{k,i}(s)-Q_{k,i}(s')\Vert_{\psi_2} &\leq \Gamma_k |s-s'|,\\
\Vert \Vert  V_k(s)-V_k(s')\Vert_\RMS \Vert_{\psi_2}&\leq \Gamma_k |s-s'|,&
\Vert\Vert U_k(s)-U_k(s')\Vert_\RMS \Vert_{\psi_2} &\leq \Gamma_k  |s-s'|.
\end{align*}
\end{lemma}
\begin{proof}
Recall that by Proposition~\ref{thm:phase-diagram-D}, we already know that for all $k\geq 0$, there exists $c\geq 0$ that only depends on $B$ and $k$ such that $h_{k,i}$ and $b_{k,i}$ are $c$-Lipschitz. Let us now prove the claim by recursion over $k$. Thanks to our choice of initialization of the dynamics, the claim is trivially satisfied at $k=0$ with $\Gamma_0=0$. 

Let us assume that the claim is true at $k-1\geq 0$ for some $\Gamma_{k-1}\geq 0$. Let us use $c$ for a scalar that only depends on $B$ and $\Gamma_{k-1}$. Then using some estimates from Proposition~\ref{thm:phase-diagram-D} $\forall s,s'\in [0,1]$,
\begin{align*}
&\Vert U_k(s)-U_k(s')\Vert_\RMS   \\
&\leq \Vert U_{k-1}(s)-U_{k-1}(s')\Vert_\RMS \\
&\qquad +\frac{\eta}{n} \sum_{i=1}^n \Vert \rho'(P_{i,k-1}(s))\clip_u(Q_{i,k-1}(s))h_{i,k-1}(s) - \rho'(P_{i,k-1}(s'))\clip_u(Q_{i,k-1}(s'))h_{i,k-1}(s')\Vert_\RMS\\
& \leq \Vert U_{k-1}(s)-U_{k-1}(s')\Vert_\RMS + c |P_{i,k-1}(s)-P_{i,k-1}(s')|\cdot |\clip_u(Q_{i,k-1}(s))|\\ 
&\qquad  + c |\clip_u(Q_{i,k-1}(s))-\clip_u(Q_{i,k-1}(s'))| + c |\clip_u(Q_{i,k-1}(s'))|\cdot \Vert h_{i,k-1}(s)-h_{i,k-1}(s')\Vert_\RMS.
\end{align*}
(Here observe that if $\clip$ was not bounded, the second term would cause a loss of integrability at each iteration).
Using  estimates from Lemma~\ref{lem:subgaussian-bis}, the recursion hypothesis, the fact that $\rho'$ is $B$-Lipschitz continuous and that $\clip$ is Lipschitz (for the third term) and bounded (for the second term), we deduce
\begin{align*}
\Vert \Vert U_k(s)-U_k(s')\Vert_\RMS\Vert_{\psi_2} \leq (\Gamma_{k-1}+ c\Gamma_{k-1}+c)|s-s'|.
\end{align*}
Analogously,
\begin{align*}
\Vert V_k(s)-V_k(s')\Vert_\RMS &\leq \Vert V_{k-1}(s)-V_{k-1}(s')\Vert_\RMS \\
&+\frac{\eta}{n} \sum_{i=1}^n \Vert \rho(P_{i,k-1}(s))b_{i,k-1}(s) - \rho(P_{i,k-1}(s'))b_{i,k-1}(s')\Vert_\RMS
\end{align*}
so using Proposition~\ref{thm:phase-diagram-D} and Lemma~\ref{lem:subgaussian-bis}, the recursion hypothesis and the fact that $\rho$ is $B$-Lipschitz continuous, we deduce
\begin{align*}
\Vert \Vert V_k(s)-V_k(s')\Vert_\RMS\Vert_{\psi_2} \leq (\Gamma_{k-1}+ c\Gamma_{k-1}+c)|s-s'|.
\end{align*}
Now, consider the preactivation random variables:
\begin{align*}
&| P_{k,i}(s)-P_{k,i}(s')| \\
&\quad  \leq \Vert U_k(s)-U_k(s')\Vert_\RMS \Vert h_{k,i}(s)\Vert_\RMS+ |\ipD{U_0}{h_{k,i}(s)-h_{k,i}(s')}|+\Vert \Delta U_k(s')\Vert_\RMS\Vert h_{k,i}(s)-h_{k,i}(s')\Vert_\RMS,
\end{align*}
hence, since $\Vert \ipD{U_0}{h_{k,i}(s)-h_{k,i}(s')}\Vert \leq D^{-1}\Vert U_0\Vert_{\psi_2}\cdot \Vert h_{k,i}(s) - h_{k,i}(s')\Vert_2$ by definition of the subgaussian vector norm,
$$
\Vert P_k(s)-P_k(s') \Vert_{\psi_2} \leq (c \Gamma_k+c) |s-s'|.
$$
We can prove $\Vert Q_k(s)-Q_k(s')\Vert_{\psi_2} \leq (c \Gamma_k+c) |s-s'|$ analogously and the claim follows by recursion.
\end{proof}

\subsection{A refined stochastic approximation result}
In this section, we prove a variant of the stochastic approximation lemma (Lemma~\ref{lem:SA}) which is used later in the proof of Theorem~\ref{claim:D-dependence}. 

\begin{lemma}[Stochastic approximation, bis]\label{lem:SA-relaxed}
Let $f:[0,1]\times \RR^D\times \RR^p\to\RR^D$ a measurable function and $(Z(s))_{s\in [0,1]}$ a $\RR^p$-valued stochastic process. 
Assume that there exists $L_x,B>0$ such that 
\begin{align}\label{eq:SA-regularity-bis}
\Vert F(s,0)\Vert_\RMS\leq B, && \Vert F(s,x)-F(s,x')\Vert_\RMS\leq L_x\Vert x-x'\Vert_\RMS,\quad \forall s\in [0,1], x,x'\in \RR^D.
\end{align}
Let $R>0$ be given by Lemma~\ref{lem:ODE-bound} (depending only on $B$, $\Vert a(0)\Vert_2$ and $L_x$) such that the unique solution of the Mean ODE
\begin{align}\label{eq:mean-ODE-lemma-bis}
a(0)\in \RR^D,&& a'(s)=F(s,a(s)),&& F(s,x)\coloneqq \E[ f(s,x,Z(s))]
\end{align}
satisfies  $\sup_{s\in [0,1]}\Vert a(s)\Vert_{\RMS} \leq R$ and $s\mapsto a(s)$ is $R$-Lipschitz in RMS norm. Moreover, we assume the following local controls:
\begin{itemize}
\item  there exists $L_s>0$ such that $\forall \Vert x\Vert_\RMS\leq R$, $\forall s,s'\in [0,1]$, $\Vert F(s,x)-F(s',x)\Vert_\RMS\leq L_s|s-s'|$.
\item there exists $K_1>0$ such that $\forall s\in [0,1]$ and $\forall \Vert x\Vert_\RMS\leq R$, the $\RR^D$-valued random variable $f(s,x,Z(s))$ is subexponential with $\Vert f(s,x,Z(s))-\E[f(s,x,Z(s))]\Vert_{\psi_1}\leq K_1$.
\end{itemize}
For integers $M,L\geq 1$, let $s_\ell\coloneqq \ell/L$ and consider the ``inexact Euler Monte-Carlo'' scheme
\begin{align}
\hat a^0\in \RR^D,&& \hat a^{\ell}=\hat a^{\ell-1}+\frac{1}{LM}\sum_{j=1}^M \hat f(s_{\ell-1},\hat a^{\ell-1},\hat Z^{j,\ell}),\quad \ell\in [1,L]
\end{align}
where $(\hat Z^{j,\ell})_{j,\ell}$ are random vectors and, there exists $\eps_0,\eps_1 \geq 0$  such that:
\begin{itemize}
\item[(i)] Bounded initial mismatch: $\|\hat a^0-a(0)\|_\RMS \le \eps_0$.
\item[(ii)] Error control: there exists $\epsilon_1,K_2(\delta_0)>0$ and a family $(Z^{j,\ell})$ of independent samples of $Z$ such that with probability at least $1-\delta_0$, $\forall x, \hat x\in \RR^D$ such that $\Vert x\Vert_\RMS,\Vert \hat x\Vert_{\RMS}\leq 2R$ and $\forall \ell\in [0:L-1]$, it holds
$$
\Big\Vert \frac1M \sum_{j=1}^M  \big( f(s_{\ell},x,Z^{j,\ell+1}(s_{\ell}))-\hat f(s_\ell,\hat x,\hat Z^{j,\ell+1}\big)\Big\Vert_\RMS \leq K_2(\delta_0) (\epsilon_1+\Vert x-\hat x\Vert_\RMS).
 $$
\end{itemize}
Then there exists and absolute $c_1>0$ such that for any $\delta\in (0,1)$ with probability at least $1-\delta-\delta_0$, it holds
\begin{align*}\label{eq:SA-lemma-subexp}
\sup_{0\leq \ell \leq L} \Vert \hat a^\ell -a(s_\ell)\Vert_\RMS 
\leq c_2 e^{K_2(\delta_0)}\Big(\epsilon_0 +K_2(\delta_0) \epsilon_1 + \frac{L_s+L_x R}{L} +K_1 \frac{1+\log(1/\delta)/\sqrt{D}}{\sqrt{ML}}\Big).
\end{align*}
provided that the right-hand side is smaller than $R$.
\end{lemma}

\begin{proof}  
For $\ell\in [0:L-1]$, it holds
\begin{align*}
a(s_{\ell+1})-\hat a_{\ell+1} &= a(s_\ell)-\hat a_\ell + \int_{s_\ell}^{s_{\ell+1}} \dot{a}(s)\d s -\frac{1}{ML}\sum_{j=1}^M \hat f(s_\ell,\hat a^\ell,\hat Z^{j,\ell+1}) \\
&= a(s_\ell)-\hat a_\ell + \underbrace{ \int_{s_\ell}^{s_{\ell+1}} \dot{a}(s)\d s-\frac1L F(s_\ell,a(s_\ell)}_{e_{euler}^{\ell+1}} \\
&\qquad + \underbrace{\Big( \frac{1}{L} F(s_\ell,a(s_\ell))- \frac1{ML} \sum_{j=1}^M f(s_\ell,a(s_\ell),Z^{j,\ell+1}(s_\ell))\Big)}_{e_{mc}^{\ell+1}}\\
&\qquad + \underbrace{\frac{1}{LM} \sum_{i=1}^M \Big( f(s_\ell,a(s_\ell),Z^{j,\ell+1}(s_\ell))-\hat f(s_\ell,\hat a^\ell,\hat Z^{j,\ell+1})\Big)}_{e_{approx}^{\ell+1}}.
\end{align*}
By recursion, we have
\[
a(s_\ell)-\hat a_\ell = a(0)-\hat a_0 + \sum_{k=1}^{\ell} e_{euler}^{k}+\sum_{k=1}^{\ell} e_{mc}^k +  \sum_{k=1}^{\ell}   e_{approx}^{k} 
\]
and therefore with $\Delta_\ell^a \coloneqq \Vert a(s_\ell)-\hat a_\ell\Vert_\RMS$, it holds
\begin{equation}\label{eq:SA-proof-bis-Delta}
\Delta_\ell^a \leq \epsilon_0 + \sum_{k=1}^{\ell} \Vert e_{euler}^{k}\Vert_\RMS+\sum_{k=1}^{\ell} \Vert e_{approx}^{k}\Vert_\RMS +  \Big\Vert \sum_{k=1}^{\ell} e_{mc}^{k}\Big\Vert_\RMS . 
\end{equation}
As in Lemma~\ref{lem:SA}, we have for $\ell \in [1:L]$,
$
\Vert e_{euler}^\ell\Vert_\RMS \leq \frac{L_s+L_x\cdot R}{2L^2}
$.
Let $L'\leq L$ be such that $\Delta_k^a\leq R$ for all $k\leq L'$ (so that $\Vert \hat a^k\Vert_\RMS\leq 2R$). Then by Assumption (ii), the term involving $e_2^k$ is bounded for $\ell\leq L'+1$ under an event of probability at least $1-\delta_0$ as
$$
\sum_{k=1}^{\ell} \Vert e_{approx}^k\Vert_\RMS \leq \frac{K_2(\delta_0)}{L}\Big(\epsilon_1 + \sum_{k=0}^{\ell-1} \Delta_\ell\Big).
$$
 In the last term, we have $e_{mc}^k=\frac{1}{L}\frac{1}{M}\sum_{j=1}^M \xi^{j,k}$ where the random variables $\xi^{j,k}$ are independent, centered and subexponential with $\Vert \xi^{j,\ell}\Vert_{\psi_1}\leq K_1$. By subexponential concentration  (Lemma~\ref{lem:subexp-concentration}) there exists an absolute constant $c>0$ such that,
with probability at least $1-\delta$ it holds
$$
\max_{1\leq \ell<L} \Big \Vert \frac{1}{LM}\sum_{k=1}^{\ell}\sum_{j=1}^M \xi^{j,k}\Big\Vert_\RMS \leq cK_1 \frac{1+\log(1/\delta)/\sqrt{D}}{\sqrt{ML}}.
$$
(Note the division by $\sqrt{D}$ which comes from the switch between $\ell_2$ to RMS norm.)
Plugging all the error estimates into~\eqref{eq:SA-proof-Delta} and by a union bound, we obtain that with probability at least $1-\delta-\delta_0$, for $\ell\in [1:L']$,
\[
\Delta_\ell^a \leq \epsilon_0 + \frac{L_s+L_x\cdot R}{2L} +cK_1 \frac{1+\log(1/\delta)/\sqrt{D}}{\sqrt{ML}}
+ \frac{K_2(\delta_0)}{L}\Big(\epsilon_1 + \sum_{k=0}^{\ell-1} \Delta_\ell\Big).
\]
The result follows for $\ell<L'$ by discrete Gronwall's lemma. If the right-hand side is smaller than $R$, then $L'=L$ and the claim follows.
\end{proof}

\subsection{Proof of Theorem~\ref{claim:D-dependence}}
Recall that $(Z^{j,\ell}_k=(U^{j,\ell}_k,V^{j,\ell}_k))_{k\geq 0}$ are iid samples from the limit dynamics satisfying $Z^{j,\ell}_0=\hat Z_0^{j,\ell}$. Let us consider the decomposition $U_k^{j,\ell}=U_0^{j,\ell}+\Delta U_k^{j,\ell}$ and $\hat U_k^{j,\ell}=U_0^{j,\ell}+\Delta \hat U_k^{j,\ell}$ and similarly for $V_k^{j,\ell}$ and $\hat V_k^{j,\ell}$, $ Z_k^{j,\ell}$ and $\hat Z_k^{j,\ell}$.

To slightly lighten the notations, we consider a single sample $x_i$ since we can deal with any number of samples via a union bound as in the proof of Theorem~\ref{thm:main}. 

 In this proof, we also introduce the notation $U^{\ell}_k\in \RR^{M\times D}$ and $V^{\ell}_k \in \RR^{M\times D}$ (without the $j$ index in the exponent) to represent the weights in one layer organized in a matrix (each row contains the weights $U^{j,\ell}_k(s_{\ell-1})\in \RR^D$ of one unit). For a matrix $U\in \RR^{M\times D}$ and $p\geq 2$, we will consider the mixed normalized $\ell_2/\ell_p$ norm
 \[
 \Vert U\Vert_{\RMS, \bar p} = \Big( \frac{1}{M} \sum_{j=1}^M \Vert U^{j}\Vert_\RMS^p\Big)^{1/p}.
 \]
(the $\bar p$-th norm is always along the dimension of length $M$).
 For $x\in \RR^D$ and $y\in \RR^M$, we will repeatedly use the inequalities
 \[
 \Vert U^\ell_k x\Vert_{\bar p} \leq D \Vert U^\ell_k\Vert_{\RMS, \bar p} \Vert x\Vert_{\RMS}, \qquad  \Vert (U^\ell_k)^\top  y\Vert_{\bar 2} \leq M \Vert U^\ell_k\Vert_{\RMS, \bar p} \Vert y\Vert_{\bar p}
 \]
 which can be deduced from $ \Vert Ux\Vert_{\bar p}^p = M^{-1}\sum_{j=1}^M |u_j^\top x|^p\leq D^p \Vert x\Vert^p_\RMS (M^{-1}\sum_j \Vert u_j\Vert_\RMS^p)$ where $u_j$ is the $j$-th row of $U$.
 We also introduce $P_k^\ell, Q_k^\ell\in \RR^M$ for the vector of activations in one layer.
 
 Throughout the proof, $c$ denotes a scalar that only depends on $B$ (from Assumption~\ref{ass:high-dim}), the exponent $p$ and $k$. We will take $p=2$ if $\clip_b$ is bounded and $p=4$ if not. We redefine $ \Delta_k^U\coloneqq \max_{\ell\in [1:L]} \Vert U^{\ell}_k-\hat U^{\ell}_k\Vert_{\bar 2,\bar p}$ and similarly for $\Delta_k^V$ and we also use the shorthand $\Delta_k^Z \coloneqq \Delta_k^U \wedge \Delta_k^V$. 

\paragraph{Step 1. Update error bound.} Let $K$ be the terminal time index. We know by Proposition~\ref{thm:phase-diagram-D} that there exists $c>0$ such that for $k\leq K$, $\forall s\in [0,1]$ $\Vert h_k(s)\Vert_\RMS\wedge \Vert b_k(s)\Vert_\RMS \leq c$.
Moreover, by Lemma~\ref{lem:subgaussian-bis}, there exists $c>0$ such that $ \Vert \Vert \Delta U^{j,\ell}_k\Vert_\RMS \Vert_{\psi_2} \leq c$ for $\ell\in[1:L]$, $k\in [0:K]$, $j\in [1:M]$. Since the random variables $U^{j,\ell}_k$ are iid across $j$, it follows from the concentration of $\ell_p$ norm (e.g.~\cite[Lemma 3.3]{sheu2023matrix}) that with probability at least $1-\delta_0$,
\[
\Vert \Delta U_k^\ell\Vert_{\bar 2, \bar p} \leq c\Big(1+\frac{\sqrt{\log(1/\delta_0)}}{M^{1/p}}\Big).
\]

Reasonning similarly for $\Delta V_k^\ell$ and by a union bound, with probability at least $1-\delta_0$, \[
\max_{\ell\in [1:L]} \Vert \Delta U^{\ell}_k \Vert_{\RMS,\bar p} \wedge \Vert \Delta V^{\ell}_k \Vert_{\RMS,\bar p} < c\Big( 1+ \frac{\sqrt{\log L}}{M^{1/p}}+ \frac{\sqrt{\log(1/\delta_0)}}{M^{1/p}}\Big)\leq c\Big( 1+ \frac{\sqrt{\log(1/\delta_0)}}{M^{1/p}}\Big)
\]
since we have assumed $\log L \leq c M^{2/p}$. 

By~\cite[Theorem 1.2]{sheu2023matrix}, and by a union bound over $L$, we also have with probability at least $1-\delta_0$ that
$$
\max_{\ell\in [1:L]} \Vert U_0^\ell\Vert_{\bar 2\to \bar p} \wedge \Vert V_0^\ell \Vert_{\bar 2\to \bar p} \leq cD \Big(1 +\frac{\sqrt{\log(1/\delta_0)}}{M^{1/p}}\Big).
$$
and also, for the transpose matrices, we have the more classical bound
$$
\max_{\ell\in [1:L]} \Vert (U_0^\ell)^\top \Vert_{\bar 2\to \bar 2} \wedge \Vert (V_0^\ell)^\top \Vert_{\bar 2\to \bar 2} \leq cM \Big(1 +\frac{\sqrt{\log(1/\delta_0)}}{\sqrt{M}}\Big).
$$

Let $\tilde K\leq K$ be the largest integer such that the following holds:
\begin{align*}
\max_{k\leq \tilde K}\Delta_k^h \wedge \Delta_k^b &\leq \max_{k\leq K}\sup_{s\in [0,1]}\Vert h_k(s)\Vert_{\RMS} \wedge \Vert b_k(s)\Vert_{\RMS} \coloneqq R,\\
\max_{k\leq \tilde K} \Delta_k^U \wedge \Delta_k^V&\leq \max_{k\leq K}\sup_{\ell \in [1:L]} \Vert \Delta U^\ell_k\Vert_{\RMS,\bar p} \wedge \Vert \Delta V^\ell_k\Vert_{\RMS,\bar p}
\end{align*} 
and let us from now on assume $k\leq \tilde K$. Later, we will ensure that it holds $\tilde K\geq K$ under the favorable event built in the proof by taking the scalar $c_2$ in the statement of the theorem small enough to conclude the proof.

For $k\leq \tilde K$, we have 
\begin{align}\label{eq:error-preactivations}
 \Vert P_k^\ell-\hat P^\ell_k\Vert_{\bar p} &\leq \frac1D \Vert (U_0^\ell+\Delta U_k^\ell)h_k^\ell - (U_0^\ell+\Delta \hat U_k^\ell)\hat h_k^\ell \Vert_{\bar p}\\
& \leq \frac{1}{D}\Vert U_0^\ell\Vert_{\bar 2 \to \bar p} \cdot \Vert h_k^\ell-\hat h_k^\ell\Vert_\RMS + \frac{D}{D}\Vert \Delta U^\ell_k \Vert_{\bar 2, \bar p} \cdot  \Vert h_k^\ell - \hat h_k^\ell\Vert_\RMS+ \frac{D}{D}\Vert \Delta U^\ell_k-\Delta \hat U^\ell_k \Vert_{\bar 2, \bar p} \cdot  \Vert  \hat h_k^\ell\Vert_\RMS
 \\
 &\leq c\Big(1+\frac{\sqrt{\log(1/\delta_0)}}{M^{1/p}}\Big) \Big( \Delta_k^h + \Delta_k^{U}\Big)
\end{align} 
and analogously
\begin{align*}
\Vert Q_k^\ell-\hat Q^\ell_k\Vert_{\bar p} \leq c\Big(1+\frac{\sqrt{\log(1/\delta_0)}}{M^{1/p}}\Big) \Big( \Delta_k^b + \Delta_k^{V}\Big). 
\end{align*}
It follows
\begin{align*}
\Vert V_{k+1}^{\ell}-\hat V_{k+1}^{\ell}\Vert_{\bar 2,\bar p} &\leq \Vert V_{k}^{\ell}-\hat V_{k}^{\ell}\Vert_{\bar 2, \bar p}+c \Vert  \rho(P_k^\ell) (b_k^\ell)^\top -  \rho(\hat P^\ell_k)(\hat b_k^\ell)^\top \Vert_{\bar 2, \bar p}\\
&\leq  \Vert V_{k}^{\ell}-\hat V_{k}^{\ell}\Vert_{\bar 2, \bar p}+ c \Vert P_k^\ell-\hat P^\ell_k\Vert_{\bar p} \cdot \Vert \hat b_k^\ell\Vert_\RMS + c \Vert b_k^\ell-\hat b_k^\ell\Vert_\RMS \cdot (1+\Vert P^\ell_k\Vert_{\bar p})\\
&\leq \Vert V_{k}^{\ell}-\hat V_{k}^{\ell}\Vert_{\bar 2, \bar p} + c\Big(1+\frac{\sqrt{\log(1/\delta_0)}}{M^{1/p}}\Big) \Big( \Delta_k^h +\Delta_k^b +\Delta_k^U \Big).
\end{align*}
By similar computations, and using in particular that $\phi:(x,y)\mapsto \rho'(x)\clip_u(y)$ is Lipschitz, we also have
\begin{align*}
\Vert U_{k+1}^{\ell}-\hat U_{k+1}^{\ell}\Vert_{\bar 2, \bar p} &\leq \Vert U_{k}^{\ell}-\hat U_{k}^{\ell}\Vert_{\bar 2, \bar p} + c\Vert  \phi(P_k^\ell, Q_k^\ell) ( h_k^\ell)^\top - \phi(\hat P_k^\ell,\hat Q_k^\ell) ( \hat h_k^\ell)^\top  \Vert_{\bar 2, \bar p}\\
&\leq  \Vert U_{k}^{\ell}-\hat U_{k}^{\ell}\Vert_{\bar 2, \bar p}+
c\Vert \phi(P_k^\ell,Q_k^{\ell})\Vert_{\bar p} \cdot 
\Vert  h^\ell_k-\hat h^\ell_k\Vert_{\bar 2}+c\Vert \hat h^\ell_k\Vert_{\bar 2}\Vert \phi(P_k^\ell,Q_k^\ell) -\phi(\hat P_k^{\ell},\hat Q_k^{\ell}\Vert_{\bar p}  \\
&\leq \Vert U_{k}^{\ell}-\hat U_{k}^{\ell}\Vert_{\bar 2,\bar p} + c \Big(1+\frac{\sqrt{\log(1/\delta_0)}}{M^{1/p}}\Big) \Big( \Delta_k^h +\Delta_k^b +\Delta_k^U +\Delta_k^V \Big).
\end{align*}

Overall, it holds by a union bound for $k\leq \tilde K$ with probability at least $1-\delta_0$,
\begin{align}\label{eq:recursion-bis-raw}
\Delta_{k+1}^Z &\leq c\Big(1+\frac{\sqrt{\log(1/\delta_0)}}{M^{1/p}}\Big)\Big( \Delta_k^Z +\Delta_k^h +\Delta_k^b \Big).
\end{align}
It is therefore sufficient to obtain controls on $\Delta_k^h$ and $\Delta_k^b$ (which are $\ell_2$ type controls) in order to conclude the proof.
We will do so by applying Lemma~\ref{lem:SA-relaxed} with the functions
\begin{align*}
f^h_{k}(s,x,(u,v))=v\rho(\ipD{u}{x}),&& f^b_{k}(s,x,(u,v))= \rho'(\ipD{u}{ h_k(s)}) \clip_b(\ipD{v}{x})u
\end{align*}
involved in the $k$-th forward pass and backward pass, respectively.

\paragraph{Step 2. Error on the forward pass} 
Let us verify the hypotheses of this lemma for the $k$-th forward pass $f^h_{k}$ where the corresponding mean ODE velocity field is
$$
F^h_{k}(s,x)=\E[V_k(s)\rho(\ipD{ U_k(s)}{x})].
$$
\begin{itemize}
\item \emph{Regularity of the Mean ODE~\eqref{eq:SA-regularity-bis}.}  Clearly, $\Vert F^h(s,0)\Vert_\RMS \leq c$ by Lemma~\ref{lem:subgaussian-bis}. The Lipschitz regularity can be shown using the regularity estimates of Lemma~\ref{lem:Lipschitz-bis} as follows:
\begin{align*}
\Vert F^h(s,x) - F^h(s',x)\Vert_{\RMS}&\leq \E[\Vert V_k(s)-V_k(s')\Vert_{\RMS}|\rho(U_k(s)^\top x/D)|]\\
&\quad + \E[\Vert (V_0+V_k(s'))(\rho(U(s)^\top x/D)-\rho(U(s')^\top x/D))\Vert_{\RMS}]\\
&\leq \E^{1/2} [\Vert V(s)-V(s')\Vert^2_{\RMS} ] \cdot \E^{1/2}[\rho(U_k(s)^\top x/D)^2]\\
&\quad +c \big({\sigma_0}/\sqrt{D}+1\big) \cdot \E^{1/2} [\Vert U(s)-U(s')\Vert_{\RMS}^2] \cdot \Vert x\Vert_{\RMS} \\
&\leq c|s-s'|(1+\Vert x\Vert_{\RMS})
\end{align*}
using in particular the decorellation Lemma~\ref{lem:scalar-vector-expectation}. The Lipschitz regularity of $F$ in $x$ can be derived similarly.
\item \emph{Subexponential fluctuations.}  Using the Lipschitz continuity of $\rho$, for $\Vert x\Vert_\RMS \leq R$, it holds
\begin{align*}
\Vert f^h_{k}(s,x,(U_k,V_k))\Vert_{\psi_1}&=\Vert V_k(s)\rho (U_k(s)^\top x/D)\Vert_{\psi_1} \\
&\leq \Vert V_k(s)\Vert_{\psi_2} \Vert \rho(U_k(s)^\top x/D)\Vert_{\psi_2}\\
&\leq c \Vert V_k(s)\Vert_{\psi_2} \Big(1+\frac{\Vert x\Vert_2}{D}\cdot \Vert U_k(s)  \Vert_{\psi_2}\Big).
\end{align*}
By the subgaussian bounds in Lemma~\ref{lem:subgaussian-bis}, we have $\Vert V_k\Vert_{\psi_2} \wedge \Vert U_k\Vert_{\psi_2}\leq c\sqrt{D}$.

Therefore, we have for $\Vert x\Vert_{\RMS}\leq R$ and $s\in [0,1]$ that $\Vert f_{k}^h(s,x,(U_k,V_k))\Vert_{\psi_1}\leq c\sqrt{D}\eqqcolon K_1$.
\item \emph{Error controls.} In the forward pass, the error in Assumption (i) of Lemma~\ref{lem:SA-relaxed} is $\epsilon_0=0$ (this error term only appears in the backward pass). Let us study the error in Assumption  (ii). Let $x,\hat x\in \RR^D$ such that $\Vert x\Vert_\RMS, \Vert \hat x\Vert_\RMS \leq 2R$ and $\ell\in [1:L']$. Using the operator norm bounds discussed before, and with $A_k^\ell = \rho(P_k^\ell)$, $\hat A_k^\ell=\rho(\hat P_k^\ell) \in \RR^M$ the activation vectors, it holds
\begin{align*}
&\Big \Vert \frac{1}{M} \sum_{j=1}^M \big(f^h_{k}(s_\ell,x,Z^{j,\ell}_k)-\hat f^h_{k}(s_\ell,\hat x,\hat Z^{j,\ell}_k)\big)\Big \Vert_\RMS \\
&= \Big \Vert \frac{1}{M} (V_0^\ell + \Delta V_k^\ell)^\top \rho(P_k^\ell) - (V_0^\ell + \Delta \hat V_k^\ell)^\top \rho(\hat P_k^\ell) \Big\Vert_\RMS\\
&\leq c \Big \Vert \frac{1}{M} (V_0^\ell)^\top (A_k^\ell-\hat A_k^\ell) \Big\Vert_\RMS+  \Big \Vert \frac{1}{M}(\Delta V_k^\ell - \Delta \hat V_k^\ell)^\top A_k^\ell \Big\Vert_\RMS+ \Big \Vert \frac{1}{M} (\Delta \hat V_k^\ell)^\top (A^{\ell}_k-\hat A^\ell_k)\Big\Vert_\RMS\\
&\leq c\Big(1+\frac{\sqrt{\log(1/\delta_0)}}{M^{1/p}}\Big) \Big(\Vert A_k^\ell - \hat A_k^\ell\Vert_{\RMS}+\Vert A_k^\ell\Vert_\RMS \Vert \Delta V_k^\ell - \Delta \hat V_k^\ell\Vert_{\bar 2, \bar 2} + \Vert \hat V_k^\ell\Vert_{\bar 2, \bar 2} \Vert A^{\ell}_k-\hat A^\ell_k\Vert_\RMS\Big)\\
&\leq c \Big(1+\frac{\sqrt{\log(1/\delta_0)}}{M^{1/p}}\Big)^2 (\Vert x-\hat x\Vert_\RMS + \Vert V_k^\ell -\hat V_k^\ell\Vert_{\RMS,\RMS} + \Vert U_k^\ell -\hat U_k^\ell\Vert_{\RMS,\RMS} )
\end{align*}
with probability at least $1-\delta_0$. Here the error $\Vert A^\ell_k-\hat A_k^\ell\Vert_\RMS$ was controlled as in~\eqref{eq:error-preactivations}. Since $\Vert V_k^\ell -\hat V_k^\ell\Vert_{\RMS,\RMS} + \Vert U_k^\ell -\hat U_k^\ell\Vert_{\RMS,\RMS}\leq \Delta_k^Z$,  Assumption (ii) holds with $K_2=c \Big(1+\frac{\sqrt{\log(1/\delta_0)}}{M^{1/p}}\Big)^2$ and $\epsilon_1=\Delta^U_k\wedge \Delta^V_k$. 
\end{itemize}
Therefore, by  Lemma~\ref{lem:SA-relaxed}, we have with probability at least $1-\delta_0-\delta'$
$$
\Delta_k^h \leq e^{c\Big(1+\frac{\sqrt{\log(1/\delta_0)}}{M^{1/p}}\Big)^2} \Big( \Big(1+\frac{\sqrt{\log(1/\delta_0)}}{M^{1/p}}\Big)^2\Delta_k^Z +\frac{1}{L} + \frac{\sqrt{D}+\log(1/\delta')}{\sqrt{ML}}\Big).
$$
Remark that at this stage we have never used the fact that we propagate controls in $\ell_p$ norms with $p$ potentially larger than $2$. This will be needed in the next step in case $\clip_b$ is not bounded.

\paragraph{Step 3. Error on the backward pass.}
Let us now verify the hypotheses of the stochastic approximation lemma for the $k$-th backward pass $f^b_{k}$ where the corresponding mean ODE velocity field is
$$
F^b_{k}(s,x)=\E[\rho'(P_k(s))\ipD{V_k(s)}{x} U_k(s)].
$$
To avoid overloading the expressions, we ignore for the moment the clipping function $\clip_b(\cdot)$ and will mention explicitly the moment in the proof where inserting this function and assuming it has a bounded range helps getting better estimates.

\begin{itemize}
\item \emph{Regularity of the Mean ODE~\eqref{eq:SA-regularity-bis}.}  Clearly, $\Vert F^b_k(s,0)\Vert_\RMS =0 \leq c$. For the Lipschitz regularity in $s$, using the decorellation Lemma~\ref{lem:scalar-vector-expectation}-(ii) in the second inequality:
\begin{align*}
&\Vert F^b_k(s,x) - F^b_k(s',x)\Vert_{\RMS}\\
& \leq \Vert \E[(\rho'(P_k(s))\ipD{V_k(s)}{x}-\rho'(P_k(s'))\ipD{V_k(s')}{x}) (U_0+\Delta U_k(s))]\Vert_\RMS \\
&\quad + \E [|\rho'(P_k(s'))\ipD{V_k(s')}{x}| \Vert \Delta U_k(s)-\Delta U_k(s')\Vert_\RMS] \\
&\leq c \E^{1/2}[|\rho'(P_k(s))\ipD{V_k(s)}{x}-\rho'(P_k(s'))\ipD{V_k(s')}{x}|^2]\\
&\quad + \E^{1/2}[ |\rho'(P_k(s'))\ipD{V_k(s')}{x}|^2] \E^{1/2}[\Vert \Delta U_k(s)-\Delta U_k(s')\Vert_\RMS^2]\\
&\leq c \Vert x\Vert_{\RMS}\Big( \E^{1/4}[|\rho'(P_k(s))-\rho'(P_k(s'))|^4]+  \E^{1/4}[\Vert \Delta V_k(s)-\Delta V_k(s')\Vert_\RMS^4]+  \E^{1/2}\Vert \Delta U_k(s)-\Delta U_k(s')\Vert_\RMS^2]\Big)\\
&\leq c|s-s'|\Vert x\Vert_{\RMS}
\end{align*}
using in particular the Lipschitz estimates in subgaussian norm (and therefore in any $L^p$ norm) from Lemma~\ref{lem:Lipschitz-bis}. The Lipschitz regularity of $F$ in $x$ can be derived similarly.
\item \emph{Subexponential fluctuations.}  Using that $\rho'$ is bounded, we have
\begin{align*}
\Vert f^b_{k}(s,x,(U_k,V_k)\Vert_{\psi_1}&=\Vert \rho' (P_k(s))\ipD{V_k(s)}{x}U_k(s)\Vert_{\psi_1} \\
&\leq c \Vert U_k(s)\Vert_{\psi_2} \Vert \ipD{V_k(s)}{x}\Vert_{\psi_2}\\
&\leq c \Vert U_k(s)\Vert_{\psi_2} \frac{\Vert x\Vert_2}{D}\cdot \Vert V_k(s)  \Vert_{\psi_2}.
\end{align*}
By the subgaussian bounds in Lemma~\ref{lem:subgaussian-bis}, we have $\Vert V_k\Vert_{\psi_2} \wedge \Vert U_k\Vert_{\psi_2}\leq c\sqrt{D}$.

Therefore, we have for $\Vert x\Vert_{\RMS}\leq R$ and $s\in [0,1]$ that $\Vert f^b_{k}(s,x,(U_k,V_k))\Vert_{\psi_1}\leq c\sqrt{D}\eqqcolon K_1$.
\item \emph{Error controls.}
The error in  Assumption (i) of Lemma~\ref{lem:SA-relaxed} is $\epsilon_0 \leq c\Vert h^L_k- \hat h^L_0\Vert_\RMS$. Let us study the error in Assumption  (ii). Let $x,\hat x\in \RR^D$ such that $\Vert x\Vert_\RMS, \Vert \hat x\Vert_\RMS \leq 2R$ and $\ell\in [1:L']$. Using the operator norm bounds discussed before, and letting $A^{j} = \rho'(P_k^{j,\ell})\cdot \ipD{V_k^{j,\ell}}{x}$, $\hat A^{j} = \rho'(\hat P_k^{j,\ell})\cdot \ipD{\hat V_k^{j,\ell}}{\hat x}$ (and similarly the vectors $A, \hat A\in \RR^M$), it holds
\begin{align*}
&\Big\Vert \frac{1}{M} \sum_{j=1}^M \big(f^b_{k}(s_\ell,x,Z^{j,\ell}_k)-\hat f^b_{k}(s_\ell,\hat x,\hat Z^{j,\ell}_k)\big)\Big \Vert_\RMS \\
&=\Big\Vert \frac1M \sum_{j=1}^M A^{j} U^{j,\ell}_k(s) - \hat A^j \hat U^{j,\ell}_k(s)\Big\Vert_\RMS\\
&\leq \Big\Vert \frac1M (U^\ell_0)^\top (A-\hat A)\Big\Vert_\RMS+ \Big\Vert \frac1M (\Delta U^\ell_k-\Delta \hat U^\ell_k)^\top A)\Big\Vert_\RMS+  \Big\Vert \frac1M (\Delta \hat U^\ell_k)^\top (A-\hat A)\Big\Vert_\RMS \\
&\leq c\Big(1+\frac{\sqrt{\log(1/\delta)}}{\sqrt{M}}\Big) \Vert A-\hat A\Vert_{\RMS} + \Vert\Delta U^\ell_k-\Delta \hat U^\ell_k\Vert_{\RMS, \RMS}\cdot \Vert A\Vert_{\RMS}+ \Vert \Delta \hat U^\ell_k\Vert_{\RMS, \RMS}\cdot \Vert A-\hat A\Vert_{\RMS}.
\end{align*}
The most challenging term is $\Vert A-\hat A\Vert_\RMS$. In case $\clip_b$ is \emph{not} bounded, we proceed as follows
\begin{align*}
&\Vert A-\hat A\Vert_\RMS\\ &\leq\Vert (\rho'(P_k^{\ell})-\rho'(\hat P_k^\ell))\odot [\ipD{V_k^{j,\ell}}{x}]_{j}\Vert_{\RMS} + \Vert \rho'(\hat P_k^\ell)\odot [\ipD{V_k^{j,\ell}}{x}-\ipD{\hat V_k^{j,\ell}}{\hat x}]_{j}\Vert_{\RMS}\\
&\leq \Vert \rho'(P_k^{\ell})-\rho'(\hat P_k^\ell)\Vert_{\bar 4} \Vert D^{-1}(V_0^\ell+\Delta V_k^\ell)x\Vert_{\bar 4}+  \Vert \rho'(\hat P_k^\ell)\Vert_{\bar 4}\Vert  [\ipD{V_k^{j,\ell}}{x}-\ipD{\hat V_k^{j,\ell}}{\hat x}]_{j}\Vert_{\bar 4}
\end{align*}
where the last line uses Cauchy-Schwartz's inequality (this is necessary for the first term; but for the second term we could alternatively use the boundedness of $\rho'$). It is because of these $\ell_4$ terms that higher order moments controls are needed through the proof. Now we can decompose further all these terms into quantities that we have already controlled and we finally obtain
\begin{align*}
\Big\Vert \frac{1}{M} \sum_{j=1}^M \big(f^b_{k}(s_\ell,x,Z^{j,\ell}_k)-\hat f^b_{k}(s_\ell,\hat x,\hat Z^{j,\ell}_k)\big)\Big \Vert_\RMS \leq c\Big(1+\frac{\sqrt{\log(1/\delta)}}{M^{1/p}}\Big)^3 \Big( \Vert x-\hat x\Vert_\RMS + \Delta_k^Z + \Delta_k^h \Big)
\end{align*}
if we take $p= 4$. In case $\clip_b$ is assumed bounded, then the Cauchy-Schwarz inequality above is not needed and $\Vert A-\hat A\Vert_\RMS$ can be controlled only in terms of RMS-norm and we can take $p=2$. In all cases, we have that Assumption (ii) holds with $K_2=c \Big(1+\frac{\sqrt{\log(1/\delta_0)}}{M^{1/p}}\Big)^3$ and $\epsilon_1=\Delta^Z_k \wedge \Delta_k^h$. 
\end{itemize}
Therefore, by Lemma~\ref{lem:SA-relaxed}, we get with probability at least $1-\delta_0- \delta'$
$$
\Delta_k^b \leq e^{c\Big(1+\frac{\sqrt{\log(1/\delta_0)}}{M^{1/p}}\Big)^3} \Big( \Vert h^L_k- \hat h^L_0\Vert_\RMS+ \Big(1+\frac{\log(1/\delta_0)}{M^{1/p}}\Big)^3(\Delta_k^Z+\Delta_k^h) +\frac{1}{L} + \frac{\sqrt{D}+\log(1/\delta')}{\sqrt{ML}}\Big).
$$

\paragraph{Step 3. Conclusion.}
Let $\delta$ be the high-probability parameter in Theorem~\ref{claim:D-dependence}. From this probability ``budget'', we take $\delta_0=e^{-M^{2/p}}/2$ so that $\frac{\sqrt{\log(1/\delta_0)}}{M^{1/p}}\leq c$ to control all the matrix norms by constants. Since we have assumed $\delta>e^{-M^{2/p}}$, the remainder $\delta-\delta_0$ is still comparable to $\delta$.
Plugging these estimates in~\eqref{eq:recursion-bis-raw} and by a union bound on $k$ and on the $n$ training samples, this leads to, with probability at least $1-\delta$, for all $k\leq \tilde K$,
\begin{align*}
\Delta_{k+1}^Z 
&\leq c\left( \Delta_{k}^Z  +\frac{1}{L} +  \frac{\sqrt{D}+\log(n/\delta)}{\sqrt{ML}}\right).
\end{align*}
By the discrete Gronwall's inequality, since $\Delta_0=0$ we get
$$
\Delta_k^Z \leq c  \left(\frac{1}{L} +  \frac{\sqrt{D}+\log(n/\delta)}{\sqrt{ML}}\right)
$$
with probability at least $1-\delta$ for $k\leq \tilde K$. Now, if this control on $\Delta_k^Z$ is small enough, this allows to ensure that $K'\geq K$ and therefore that this bound holds for $k\leq K$. This concludes the proof.

\newpage
\appendix

\bibliography{LC.bib}

\appendix

\section{Appendix: subgaussian and subexponential random variables}
\label{sec:subgaussian}
Let us recall some standard tools to control tails of random variables. For a real random variable $X$ and $\theta\geq 1$, we define the norm 
$$
\Vert X\Vert_{\psi_{\theta}} \coloneqq \inf \{ t>0\;;\; \E[\exp((|X|/t)^\theta)]\leq 2\}
$$
and for a $\RR^d$-valued random vector $Y$, we define 
$$
\Vert Y\Vert_{\psi_{\theta}} \coloneqq \sup_{\substack{v\in \RR^d\\ \Vert v\Vert_2\leq 1}} \Vert v^\top Y \Vert_{\psi_\theta}.
$$
When $\theta=2$, this is called the subgaussian norm and when $\theta=1$, the subexponential norm. If $\Vert Y\Vert_{\psi_2}<+\infty$ we say that $Y$ is subgaussian and if $\Vert Y\Vert_{\psi_1}<+\infty$ we say that $Y$ is subexponential. Remark that the constant random variable equal to $1$ satisfies $\Vert 1\Vert_{\psi_\theta}=(\log(2))^{-1/\theta}$ which is smaller than $2$ for $\theta\in \{1,2\}$.

We will use the following facts about these norms:
\begin{enumerate}
\item A real-valued random variable $X$ is subgaussian iff $X^2$ is subexponential and $\Vert X\Vert_{\psi_2}^2=\Vert X^2\Vert_{\psi_1}$ (follows from the definition, or see~\cite[Lemma.~2.7.6]{vershynin2018high}). More generally, if $X,X'$ are scalar subgaussian then $\Vert X\cdot X'\Vert_{\psi_1}\leq \Vert X\Vert_{\psi_2}\cdot\Vert X'\Vert_{\psi_2}$.
\item If $X$ is a subgaussian in $\RR^D$ then $\Vert X\Vert_{\psi_2}\leq \Vert \Vert X\Vert_2\Vert_{\psi_2}\leq \sqrt{D}\Vert X\Vert_{\psi_2}$. The first inequality follows from the definition and the second follows from
\begin{align*}
\Vert \Vert X\Vert_2 \Vert^2_{\psi_2} = \Vert\Vert X\Vert_2^2\Vert_{\psi_1}=\Big\Vert \sum_{i=1}^D X_i^2\Big\Vert_{\psi_1}\leq \sum_{i=1}^D \Vert X_i^2\Vert_{\psi_1}= \sum_{i=1}^D \Vert X_i\Vert_{\psi_2}^2\leq D\Vert X\Vert^2_{\psi_2}.
\end{align*}
\item If $X$ is subgaussian in $\RR^D$ and $f:\RR^D\to \RR^D$ satisfies $\Vert f(x)\Vert_2\leq A+B\Vert x\Vert_2$ then $f(X)$ is subgaussian and $\Vert f(X)\Vert_{\psi_2}\leq 2A+B\sqrt{D}\Vert X\Vert_{\psi_2}$. This follows from
$$
\Vert f(X)\Vert_{\psi_2}\leq \Vert \Vert f(X)\Vert_2\Vert_{\psi_2}\leq \Vert A+B\Vert X\Vert_2\Vert_{\psi_2}\leq \frac{A}{\sqrt{\log(2)}}+B\Vert \Vert X\Vert_2\Vert_{\psi_2}.
$$
\item  If $X$ is subgaussian in $\RR^D$ then $\Vert X-\E[X]\Vert_{\psi_2}\leq c\Vert X\Vert_{\psi_2}$ for some absolute $c>0$ (for $X$  scalar this is~\cite[Lemma~2.6.8]{vershynin2018high}).
\end{enumerate}
For a subgaussian random vector $X$ in $\RR^D$, we also consider the variance proxy seminorm defined as
$$
\Vert X\Vert_{vp}\coloneqq \inf \Big\{ s>0\;;\; \E[ e^{u^\top (X-\E[X])}]\leq e^{s^2\Vert u\Vert^2_2/2}, \forall u\in \RR^D\Big\}.
$$
This is the infimum of all variance proxies of $X$ (defined in~\eqref{eq:variance-proxy}). As a seminorm, this quantity satisfies positive homogeneity, triangle inequality, and is nonnegative but does not separate points. There exists absolute constants $c,c'>0$ such that for any $X$ subgaussian in $\RR^D$ it holds $c\Vert X\Vert_{vp}\leq \Vert X-\E[X]\Vert_{\psi_2}\leq c'\Vert X\Vert_{vp}$. Moreover, if $f:\RR^D\to \RR^D$ is $L$-Lipschitz then 
\begin{align*}
\Vert f(X)\Vert_{vp} &\leq c \Vert f(X)-f(\E[X])+f(\E[X])-\E f(X)\Vert_{\psi_2}\\
&\leq 2cL\Vert \Vert X-\E[X]\Vert_2\Vert_{\psi_2}\leq 2c'L\sqrt{D} \Vert X\Vert_{vp}
\end{align*}

Subgaussian vectors satisfy the following classical concentration bound which can be obtained via an $\epsilon$-net argument on the sphere and a union bound.
\begin{lemma}[Subgaussian vector concentration]\label{lem:subgaussian-vector-concentration}
Let $X$ be a centered subgaussian random vector in $\RR^D$ with variance proxy $\sigma^2$. Then for all $\delta >0$, it holds with probability at least $1-\delta$
\[
\Vert X\Vert_2  \leq c\sigma (\sqrt{D}+\sqrt{\log(2/\delta)}) 
\]
where $c>0$ is an absolute constant.
\end{lemma}

We also need the concentration inequality for subexponential random variables (for convenience of the reader, we also give a proof in the precise form that we need as it is somewhat less classical than Lemma~\ref{lem:subgaussian-vector-concentration}).

\begin{lemma}[Subexponential vector concentration]\label{lem:subexp-concentration}
Let $(\xi^{j,\ell})_{j\in [1:M], \ell\in [1:L]}$ be a family of independent and centered subexponential random variables in $\RR^D$ and $K>0$ such that $\Vert \xi^{j,\ell}\Vert_{\psi_1}\leq K$. Assume that $D\leq ML$. Then there exists an absolute constant $c>0$ such that  with probability at least $1-\delta$ it holds
$$
\max_{1\leq \ell<L} \Big \Vert \frac{1}{LM}\sum_{k=1}^{\ell}\sum_{j=1}^M \xi^{j,k}\Big\Vert_2\leq cK \frac{\sqrt{D}+\log(1/\delta)}{\sqrt{ML}}.
$$
\end{lemma}
\begin{proof}
 By a usual $\epsilon$-net argument~\cite[Corollary 4.2.13]{vershynin2018high} with $\epsilon=\frac12$, it holds
\begin{align}
\P\Big( \Big\Vert \frac{1}{LM}\sum_{k=1}^{\ell}\sum_{j=1}^M \xi^{j,k}\Big\Vert_2>t\Big)\leq 5^D\max_{\Vert \lambda\Vert_2\leq 1}\P\Big( \frac{1}{LM}\sum_{k=1}^{\ell}\sum_{j=1}^M \lambda^\top \xi^{j,k} >t/2\Big) 
\end{align}
Now, for any $\lambda\in \RR^D$ with $\Vert \lambda\Vert_2= 1$, Bernstein's concentration inequality~\cite[Corollary~2.8.3]{vershynin2018high} yields, for some absolute constant $c>0$,
\begin{align}
\P\Big(\frac{1}{LM}\sum_{k=1}^{\ell}\sum_{j=1}^M \lambda^\top \xi^{j,k} >t \Big)\leq \exp\Big(-c \min \Big\{\frac{t^2}{K^2},\frac{t}{K}\Big\}ML\Big).
\end{align}
It follows that we can guarantee
$
\P\Big( \Big\Vert \frac{1}{LM}\sum_{k=1}^{\ell}\sum_{j=1}^M \xi^{j,k}\Big\Vert_2>t\Big)<e^{-s}
$
if $s$ and $t$ are such that
$$
cML \min \Big\{ \frac{t^2}{K^2},\frac{t}{K}\Big\}\geq D+s.
$$
This relation is satisfied for $t=c'K \left(\sqrt{\frac{D+s}{ML}}+\frac{D+s}{ML}\right) \leq c'' K \Big(\frac{\sqrt{D}}{\sqrt{ML}}+\frac{\sqrt{s}}{\sqrt{ML}}+\frac{s}{ML}\Big)$ using $D\leq LM$.
Then the claim follows by Lemma~\ref{lem:levy} and simplifying the expression.
\end{proof}

The next lemma shows that the tail of maximum of a running sum of independent variables is comparable to the tail of a single partial sum.
\begin{lemma}[Lévy-Ottaviani inequality]\label{lem:levy}~\cite[Proposition~1.1.2]{de2012decoupling}
Let $X_1,\dots X_L \in \RR$ be independent random variables (not necessarily centered). Then for all $t>0$, 
$$
\P\Big( \max_{1\leq k\leq L}\Big\Vert \sum_{i=1}^k X_i\Big\Vert>t \Big)\leq 3 \max_{1\leq k\leq L} \P\Big( \big\Vert \sum_{i=1}^k X_i\big\Vert>t/3 \Big).
$$
\end{lemma}

\section*{Differences between the versions of the manuscript}
\begin{itemize}
\item Version 1 (Sept.~12, 2025) is the first upload of the paper.
\item Version 2 (Mar.~2, 2026): 
\begin{itemize}
    \item the proof of the theorems have been streamlined and some minor errors have been fixed. Most importantly, we have added the ``clipping'' functions in the statement of Theorem~\ref{claim:D-dependence}, which were missing in v1.
    \item we have converted the phase diagram section (Section~\ref{sec:2LP-phase-diagram}) from a formal theorem statement into a more readable and reader friendly discussion.
    \item We have introduced the terminology ``Maximum Local Update'' (MLU) regime, while v1 was using inconsistent terminology.
\end{itemize}
\end{itemize}

\end{document}